\documentclass{article} 
\usepackage{collas2026_conference,times}

\usepackage{graphicx}
\usepackage{subcaption} 

\usepackage{amsmath, amssymb, amsthm, bm}
\usepackage{subcaption} 
\usepackage{algorithm} 
\usepackage{algorithm}
\usepackage{algpseudocode}
\usepackage{booktabs}
\usepackage{subcaption}

\usepackage[hidelinks]{hyperref}

\usepackage{caption}

\usepackage{algorithm}
\usepackage{algpseudocode}
\usepackage{amsthm}

\usepackage{easyReview}

\newtheorem{theorem}{Theorem}[section]
\newtheorem{lemma}[theorem]{Lemma}
\newtheorem{proposition}[theorem]{Proposition}

\usepackage{amsmath,amsfonts,bm}

\def\eqref#1{equation~\ref{#1}}

\def\1{\bm{1}}

\DeclareMathAlphabet{\mathsfit}{\encodingdefault}{\sfdefault}{m}{sl}
\SetMathAlphabet{\mathsfit}{bold}{\encodingdefault}{\sfdefault}{bx}{n}

\usepackage{hyperref}
\hypersetup{
    colorlinks=true,
    linkcolor=red,
    filecolor=magenta,
    urlcolor=blue,
    citecolor=purple,
    pdftitle={Overleaf Example},
    pdfpagemode=FullScreen,
    }

\title{When Geometry Aligns: Dihedral Hidden-State Transformations in UNet, ViT, and DiT Architectures.}

\author{Mojtaba Faramarzi\\
University of Montreal, Canada\\
Mila - Quebec AI Institute\\
\And 
Alex Lamb $^{*}$ \\
Tsinghua University \\
China \\
\And 
Irina Rish \thanks{Equal Supervision.} \\
University of Montreal, Canada\\
Mila - Quebec AI Institute\\
}

\collasfinalcopy 

\begin{document}

\maketitle

\begin{abstract}
Diffusion architectures now encompass convolutional UNets as well as transformer-based designs such as Diffusion Transformers (DiTs), inspired by Vision Transformers (ViTs), yet the effects of structured geometric perturbations within these architectures remain poorly understood. We study this question through a unified framework that applies reflection-based elements of the dihedral group to intermediate hidden states as controlled internal interventions, contrasting geometrically consistent and inconsistent variants. Using activation-level diagnostics, including Self-Consistency Shift (SCS), Activation Mass Scatter (AMS), and Drift, we analyze feature stability and geometric drift. We find that consistent transformations improve stability, while inconsistent ones induce predictable, architecture-specific failures. In the main Stable Diffusion 2.1 U-Net study, we evaluate seven intervention modes over three seeds and complement the internal diagnostics with image-level FID, KID, CLIP score, and LPIPS diversity. Taken together with supporting ViT and controlled DiT analyses, these results establish geometric consistency as a key principle for stable hidden-state interventions in spatially structured vision and diffusion models.
\end{abstract}

\section{Introduction}

Modern machine learning systems are increasingly modified during or after
training through fine-tuning, parameter-efficient adaptation, internal editing,
feature-level perturbations, and other forms of hidden-state intervention
\citep{vaswani2017attention,hu2021lora,meng2022rome}. In vision and generative
modeling, such interventions are applied to architectures whose intermediate
representations carry strong spatial structure, including convolutional U-Nets
and transformer-based models such as ViTs and DiTs
\citep{dosovitskiy2020vit,touvron2021deit,peebles2023dit,ho2020ddpm,song2020ddim,song2021score,rombach2022ldm}.
This raises a basic but underexplored question: what makes an internal
geometric transformation compatible with the computation of a spatially
structured model?

Prior work has studied input and latent-space augmentation and robustness
\citep{shorten2019survey,zhang2018mixup,yun2019cutmix,faramarzi2021patchup},
architectural equivariance and symmetry-aware modeling
\citep{cohen2016gcnn,weiler2018steerable,finzi2020lieconv,hoogeboom2022equivariant},
and hidden-state manipulation for interpretation, control, and editing
\citep{geva2021transformerediting,dai2021knowledgeNeurons,hertz2022prompt,meng2022rome,chefer2021transformer,bau2017network}.
However, modern diffusion and vision backbones have not been systematically
studied from the perspective of \emph{hidden-state geometric interventions}:
structured transformations applied directly to intermediate representations
inside otherwise standard architectures. This setting is especially informative
because it reveals whether a model can accommodate internal geometric change
while preserving the spatial organization of its features.

In this work, we study dihedral hidden-state transformations as controlled
internal interventions in U-Net-, ViT-, and DiT-style models. Our
contributions are fourfold. First, we introduce a unified framework for
applying hidden-state geometric transformations across convolutional and
transformer-based architectures. Second, we develop a theoretical framework
that formalizes hidden-state geometric interventions and explains geometric
consistency as the condition under which such interventions preserve coherent
computation. Third, we introduce activation-level diagnostics that we term
Self-Consistency Shift (SCS), Activation Mass Scatter (AMS), and Drift,
together with a fidelity-oriented proxy metric. Fourth, empirically, we
conduct a main quantitative study on the Stable Diffusion 2.1 U-Net,
complemented by supporting ViT and controlled DiT analyses, and show that
geometrically consistent interventions improve internal stability whereas
inconsistent interventions induce structured misalignment.

Using these diagnostics, our main U-Net study shows that geometrically
consistent interventions improve feature stability, whereas inconsistent
interventions lead to systematic geometric mismatch and reduced fidelity.


\section{Method: Geometric Hidden-State Interventions}
\label{sec:method}

We study \emph{hidden-state interventions} that apply a spatial transformation
directly to an intermediate representation inside a vision or diffusion
architecture. Let $z^{(\ell)}$ denote the hidden state at layer or block
$\ell$, and let $T$ denote a spatial transformation. Although the dihedral group $D_4$ contains both rotations and reflections, the
implemented transform set in this work is the reflection subset
$
\mathcal{T}
=
\{T_{\mathrm{hor}},T_{\mathrm{ver}},T_{\mathrm{diag}},T_{\mathrm{anti}}\}
\subset D_4,
$
namely horizontal, vertical, main-diagonal, and anti-diagonal flips. Accordingly, our empirical claims apply to this reflection-based instantiation;
extending the analysis to rotations is left for future work. This choice
isolates the effect of internal geometric change while keeping the underlying
architecture fixed. Our central question is whether such an intervention
preserves the spatial frame expected by the downstream computation.

Our organizing principle is \emph{geometric consistency}. An intervention is
\emph{consistent} if all interacting components that consume the transformed
representation operate in the same spatial frame. It is \emph{inconsistent} if
the transformation is applied to only part of a coupled computation, so that
different branches operate in incompatible coordinate systems. This principle
is the main lens through which we analyze both attention-based architectures
and U-Net-style feature hierarchies. In multi-head attention, the coupled
computation arises from the interaction of token routing within each head and
the subsequent mixing of heads through output projection
\citep{vaswani2017attention}. In U-Net backbones, it appears through
skip-connected encoder--decoder fusion \citep{ronneberger2015unet} and
attention blocks embedded inside diffusion architectures
\citep{ho2020ddpm,rombach2022ldm}. In both settings, hidden-state interventions
remain well-posed only when all coupled pathways stay geometrically aligned.

\subsection{Transformer Intervention: Flipped-Head Attention}

Consider a transformer block with input token matrix
$
X \in \mathbb{R}^{N \times d_{\mathrm{model}}},
$
where $N=P^2$ tokens are arranged on a $P\times P$ patch grid, as in Vision
Transformers \citep{dosovitskiy2020vit}. For head $h$, standard self-attention
is \citep{vaswani2017attention}
\begin{equation}
Q_h = XW_h^Q,\quad
K_h = XW_h^K,\quad
V_h = XW_h^V,\quad
O_h = \alpha_h V_h,\quad
\alpha_h = \mathrm{softmax}\!\left(\frac{Q_hK_h^\top}{\sqrt{d_k}}\right).
\label{eq:std_attn_main}
\end{equation}
The multi-head output is
\[
Y=\mathrm{Concat}(O_1,\dots,O_H)W^O.
\]

We intervene on a selected head $h^\star$ by reshaping its output into
\[
\mathcal{O}_{h^\star}\in\mathbb{R}^{P\times P\times d_v},
\qquad
\mathcal{O}_{h^\star}[i,j,:]=O_{h^\star}[t(i,j),:],
\]
where $t(i,j)=(i-1)P+j$, applying a spatial transformation $T$, and then
flattening back. For example, for a horizontal flip,
\begin{equation}
(T\mathcal{O}_{h^\star})[i,j,:]
=
\mathcal{O}_{h^\star}[i,P-j+1,:].
\label{eq:hflip_main}
\end{equation}
Let $\widetilde{O}_{h^\star}$ denote the flattened transformed output. The
intervened multi-head representation is then
\begin{equation}
\widehat{O}
=
\mathrm{Concat}(O_1,\dots,O_{h^\star-1},\widetilde{O}_{h^\star},
O_{h^\star+1},\dots,O_H),
\qquad
Y=\widehat{O}W^O.
\label{eq:flipped_head_main}
\end{equation}
A detailed derivation is provided in
Appendix~\ref{app:vit_flip_derivation}.

\paragraph{Attention-Consistent and Attention-Inconsistent Variants.}
Let
$
\mathcal{A}(Q,K,V)
=
\mathrm{softmax}\!\left(\frac{QK^\top}{\sqrt{d_k}}\right)V
$
denote the attention output of a single head, and let $\Pi_T$ be the
permutation matrix induced by the spatial transformation $T$ on token indices.
A geometrically consistent transformation of a single head satisfies
\begin{equation}
\mathcal{A}(\Pi_T Q,\Pi_T K,\Pi_T V)
=
\Pi_T\,\mathcal{A}(Q,K,V),
\label{eq:single_head_equiv_main}
\end{equation}
so token routing and aggregation remain in a common transformed frame.

In our setting, the relevant inconsistency arises not at the level of an
isolated head, but at the level of the full multi-head module. If only one
selected head $h^\star$ is transformed after attention, the module output
becomes
\begin{equation}
Y_{\mathrm{inc}}
=
\mathrm{Concat}(O_1,\dots,\Pi_T O_{h^\star},\dots,O_H)W^O,
\label{eq:attn_inconsistent_main}
\end{equation}
which mixes transformed and untransformed heads before output projection. In
contrast, a geometrically consistent transformation applies the same spatial
map to all interacting heads:
\begin{equation}
Y_{\mathrm{cons}}
=
\mathrm{Concat}(\Pi_T O_1,\dots,\Pi_T O_H)W^O
=
\Pi_T\,\mathrm{Concat}(O_1,\dots,O_H)W^O.
\label{eq:attn_consistent_main}
\end{equation}
The second equality holds because $\Pi_T$ acts on the token axis, whereas
$W^O$ acts only on channel features. Thus, the crucial distinction is not
whether a flip is applied, but whether the interacting heads remain expressed
in a common geometric frame. In iterative transformer denoisers such as DiT,
this mismatch can compound across denoising steps and manifest as increased
drift and reduced stability. Full proofs and the timestep-sensitivity analysis
are deferred to Appendix~\ref{app:qkv_appendix}.

\subsection{Random Hidden-State Geometric Augmentation}

The transformer construction above motivates a broader intervention scheme:
rather than applying a fixed transformation to a single head, we consider
stochastic hidden-state transformations applied at arbitrary intermediate
locations. Let
$
\mathcal{T}
=
\{T_{\mathrm{hor}},T_{\mathrm{ver}},T_{\mathrm{diag}},T_{\mathrm{anti}}\},
$
where the elements denote horizontal, vertical, main-diagonal, and
anti-diagonal reflections. For a selected block $\ell$ with hidden state
$H^{(\ell)}$, we sample an intervention pair $(\ell,\tau)$ with
$\tau\in\mathcal{T}$ and apply
\begin{equation}
\widetilde{H}^{(\ell)}=\tau(H^{(\ell)}).
\end{equation}
In our implementation, exactly one intervention location is selected per
mini-batch. This yields a lightweight hidden-state augmentation scheme that
exposes the model to multiple geometric views while remaining compatible with
pretrained backbones and requiring no architectural redesign. The same
consistency principle applies at this more general level: once a hidden state
is transformed, all downstream computations that interact through that state
must remain in a common spatial frame. We next make this explicit for diffusion
models.

\subsection{Hidden-State Intervention in Diffusion Models}

We now instantiate the same principle for diffusion denoisers, including both
U-Net-based architectures
\citep{ronneberger2015unet,ho2020ddpm,rombach2022ldm}
and transformer-based denoisers such as DiT \citep{peebles2023dit}. In both
cases, we apply a spatial transformation to an intermediate hidden state during
denoising and train the model to remain stable under this intervention.

Let $x_0 \sim p_{\mathrm{data}}$ be a clean training sample,
$\varepsilon \sim \mathcal{N}(0,I)$, and $x_t$ the noisy sample at diffusion
step $t$. Let $\varepsilon_\theta(x_t,t)$ denote the denoiser. For a chosen
layer or block $\ell$, let
$
F^{(\ell)}(x_t,t)\in\mathbb{R}^{C_\ell\times H_\ell\times W_\ell}
$
denote the corresponding hidden activation. Given a sampled transformation $\tau\in\mathcal{T}$, we define the transformed hidden state
\begin{equation}
\widetilde{F}^{(\ell)}(x_t,t)
=
\tau\!\left(F^{(\ell)}(x_t,t)\right).
\label{eq:diff_hidden_transform_main}
\end{equation}

Let $\varepsilon_\theta^{(\tau,\ell)}(x_t,t)$ denote the denoiser output obtained
by running the network up to block $\ell$, replacing
$F^{(\ell)}(x_t,t)$ by $\widetilde{F}^{(\ell)}(x_t,t)$, and propagating the
modified activation through the remainder of the network. The corresponding
training objective is
\begin{equation}
\mathcal{L}_{\mathrm{aug}}(\theta)
=
\mathbb{E}_{x_0,\varepsilon,t,\ell,\tau}
\Big[
\|\varepsilon-\varepsilon_\theta^{(\tau,\ell)}(x_t,t)\|_2^2
\Big].
\label{eq:diff_aug_loss_main}
\end{equation}

This objective is shared across both U-Net and DiT backbones; what differs is
how geometric consistency is realized by the underlying architecture.
In a U-Net, the critical interaction arises at skip-connected
encoder--decoder fusion. Consider a skip pair $(\ell,\ell')$ whose encoder and
decoder features have the same spatial resolution. Let
$
F_{\mathrm{enc}}^{(\ell)}\in\mathbb{R}^{C_{\mathrm{enc}}\times H\times W}
$
denote the encoder feature map, and let
$
F_{\mathrm{dec}}^{(\ell')}\in\mathbb{R}^{C_{\mathrm{dec}}\times H\times W}
$
denote the corresponding decoder feature map. Let
\(\phi\) denote the skip-fusion operator that combines these two tensors into a
single representation; for example, \(\phi\) may be concatenation followed by a
shared convolution, or any other fusion map used by the architecture. We write
the fused feature as
\begin{equation}
G^{(\ell,\ell')}
=
\phi\!\left(
F_{\mathrm{dec}}^{(\ell')},
F_{\mathrm{enc}}^{(\ell)}
\right).
\label{eq:skip_fusion_main}
\end{equation}

A geometrically consistent intervention applies the same transformation to both
coupled branches before fusion:
\begin{equation}
\widetilde{G}^{(\ell,\ell')}
=
\phi\!\left(
\tau(F_{\mathrm{dec}}^{(\ell')}),
\tau(F_{\mathrm{enc}}^{(\ell)})
\right).
\label{eq:skip_consistent_main}
\end{equation}
By contrast, transforming only one branch, for example
\[
\phi\!\left(
\tau(F_{\mathrm{dec}}^{(\ell')}),
F_{\mathrm{enc}}^{(\ell)}
\right)
\quad\text{or}\quad
\phi\!\left(
F_{\mathrm{dec}}^{(\ell')},
\tau(F_{\mathrm{enc}}^{(\ell)})
\right),
\]
creates a skip inconsistency by fusing features expressed in different spatial
frames.

In DiT-style transformer denoisers, the same objective in
Eq.~\ref{eq:diff_aug_loss_main} applies, but the coupled computation is
formed by attention and token mixing rather than skip fusion. As in the
transformer case above, geometric consistency requires all interacting
attention pathways to share the same transformed frame; partial
transformations instead create cross-head or cross-branch mismatch. Thus, the
same geometric-consistency principle governs both U-Net and DiT diffusion
models: hidden-state interventions remain stable only when all coupled
computations stay geometrically aligned. Additional implementation details are
deferred to Appendix~\ref{app:diffusion_details}.
\section{Theoretical Analysis: Stability and Regularization Under Geometric Consistency}
\label{sec:theory}
The method section defines how hidden-state transformations are applied, but
not why consistent interventions remain stable while inconsistent ones induce
mismatch. This section addresses that gap by characterizing flipped-head
representations, formalizing consistency in attention and U-Net fusion, and
giving an idealized regularization interpretation of symmetry-consistent
interventions.

\subsection{Representation Induced by a Flipped Head}

We begin with the simplest symmetric geometric intervention considered in the
paper: a horizontal flip applied to the output of a selected attention head.
The following result characterizes the spatial correspondence induced by this
operation.

\begin{lemma}[Flip-induced spatial correspondence]
\label{lem:flip_correspondence_main}
Let $h^\star$ be the selected head, and let
$\widetilde{O}_{h^\star}$ denote the flattened output obtained by horizontally
flipping the reshaped head output. Then, for every spatial position $(i,j)$,
\begin{equation}
\widetilde{O}_{h^\star}[t(i,j),:]
=
O_{h^\star}[t(i,P-j+1),:].
\label{eq:flip_correspondence_main}
\end{equation}
\end{lemma}
\noindent\emph{Proof.} See Appendix~\ref{app:proof_flip_correspondence}.

We next expand how this intervention appears in the final multi-head token
representation.

\begin{proposition}[Final token representation with a flipped head]
\label{prop:final_token_main}
Let $W^O=[W_1^O\;\cdots\;W_H^O]$ be partitioned into head-wise blocks, with
$W_h^O\in\mathbb{R}^{d_{\mathrm{model}}\times d_v}$, and let
$o_h[t]:=O_h[t,:]^\top\in\mathbb{R}^{d_v}$ denote the token feature of head
$h$ at token index $t$. Then, for each spatial location $(i,j)$,
\begin{equation}
\mathbf{y}_{t(i,j)}
=
\sum_{h\neq h^\star} W_h^O\,o_h[t(i,j)]
+
W_{h^\star}^O\,o_{h^\star}[t(i,P-j+1)].
\label{eq:final_token_theory_main}
\end{equation}
\end{proposition}
\noindent\emph{Proof.} See Appendix~\ref{app:proof_final_token}.

Proposition~\ref{prop:final_token_main} shows that the flipped-head
intervention injects mirrored contextual information into the token
representation while leaving all remaining heads anchored at the original
spatial location.

\subsection{Consistency Principles for Attention and U-Net Fusion}

We next formalize why consistency is necessary inside attention. At the level
of an isolated head, a shared spatial transformation preserves equivariance.
The inconsistency in our setting arises instead at the level of the full
multi-head module: if only part of the interacting computation is transformed,
the output projection mixes representations expressed in different geometric
frames.

\begin{proposition}[Attention consistency under a shared spatial transformation]
\label{prop:attn_consistency_main}
Let
$
Y=\mathrm{Concat}(O_1,\dots,O_H)W^O,
$
and let $\Pi_T$ denote the permutation induced by a spatial transformation $T$
on token indices. If all head outputs are transformed by the same $\Pi_T$, then
\begin{equation}
Y_{\mathrm{cons}}
=
\mathrm{Concat}(\Pi_T O_1,\dots,\Pi_T O_H)W^O
=
\Pi_T Y.
\label{eq:attn_consistency_main_theory}
\end{equation}
By contrast, if only a strict subset of heads is transformed, then in general
the resulting module output cannot be written as $\Pi_T Y$.
\end{proposition}
\noindent\emph{Proof.} See Appendix~\ref{app:proof_attn_consistency}.

An analogous principle holds for U-Net feature fusion, where the interacting
paths are the branches entering a skip-connected fusion operator.

\begin{proposition}[Skip consistency under simultaneous transformation]
\label{prop:skip_consistency_main}
Consider a skip pair $(\ell,\ell')$ at matched spatial resolution, with
encoder feature map
$
F_{\mathrm{enc}}^{(\ell)}\in\mathbb{R}^{C_{\mathrm{enc}}\times H\times W}
$
and decoder feature map
$
F_{\mathrm{dec}}^{(\ell')}\in\mathbb{R}^{C_{\mathrm{dec}}\times H\times W}.
$
Let
$
G^{(\ell,\ell')}
=
\phi\!\left(F_{\mathrm{dec}}^{(\ell')},F_{\mathrm{enc}}^{(\ell)}\right)
$
be the fused representation, where
$
\phi:\mathbb{R}^{C_{\mathrm{dec}}\times H\times W}
\times
\mathbb{R}^{C_{\mathrm{enc}}\times H\times W}
\to
\mathbb{R}^{C_{\mathrm{out}}\times H\times W}
$
is a skip-fusion operator. Assume that $\phi$ is equivariant under
simultaneous spatial permutation, i.e.
\[
\phi(\Pi_T a,\Pi_T b)=\Pi_T\phi(a,b).
\]
Then transforming both branches by the same $\Pi_T$ preserves a common
transformed frame:
\begin{equation}
\phi(\Pi_T F_{\mathrm{dec}}^{(\ell')},\Pi_T F_{\mathrm{enc}}^{(\ell)})
=
\Pi_T\,\phi(F_{\mathrm{dec}}^{(\ell')},F_{\mathrm{enc}}^{(\ell)}).
\label{eq:skip_consistency_main_theory}
\end{equation}
Transforming only one branch generally does not preserve this relation and thus
introduces geometric mismatch at fusion.
\end{proposition}
\noindent\emph{Proof.} See Appendix~\ref{app:proof_skip_consistency}.

Together, Propositions~\ref{prop:attn_consistency_main}
and~\ref{prop:skip_consistency_main} explain the mismatch patterns observed in
the experiments: stable interventions preserve geometric alignment, whereas
partial transformations introduce structured inconsistency across interacting
computations.

\subsection{Symmetry Constraints as Regularization}

Beyond stability, geometric consistency also admits an idealized statistical
interpretation. Let $\mathcal{H}$ be a hypothesis class of predictors
$h:\mathcal{X}\to\mathbb{R}$, and let $T:\mathcal{X}\to\mathcal{X}$ be a
symmetry transformation. We define the symmetry-constrained subclass
$
\mathcal{H}_{\mathrm{sym}}
=
\{h\in\mathcal{H}:\; h(x)=h(T(x))\ \text{for all }x\in\mathcal{X}\}.
$
By construction,
$
\mathcal{H}_{\mathrm{sym}}\subseteq\mathcal{H}.
$
Now let $S=\{x_i\}_{i=1}^n$ be a sample, and let
$\sigma_1,\dots,\sigma_n\overset{\mathrm{i.i.d.}}{\sim}\mathrm{Unif}\{-1,+1\}$
denote Rademacher random variables. The empirical Rademacher complexity of a
scalar-valued function class $\mathcal{F}$ is
\begin{equation}
\widehat{\mathcal{R}}_S(\mathcal{F})
=
\mathbb{E}_{\sigma}
\left[
\sup_{f\in\mathcal{F}}
\frac{1}{n}\sum_{i=1}^n \sigma_i f(x_i)
\right].
\label{eq:rademacher_main}
\end{equation}
Since $\mathcal{H}_{\mathrm{sym}}\subseteq\mathcal{H}$, monotonicity gives
$
\widehat{\mathcal{R}}_S(\mathcal{H}_{\mathrm{sym}})
\le
\widehat{\mathcal{R}}_S(\mathcal{H}).
$
To connect this to generalization, let
$
\psi:\mathbb{R}\times\mathcal{Y}\to[0,1]
$
be a loss function that is $L$-Lipschitz in its first argument, and define the
induced loss class
\[
\mathcal{L}_{\mathcal{H}}
=
\{(x,y)\mapsto \psi(h(x),y): h\in\mathcal{H}\}.
\]
Since $\mathcal{H}_{\mathrm{sym}}\subseteq\mathcal{H}$, we also have
$
\mathcal{L}_{\mathcal{H}_{\mathrm{sym}}}
\subseteq
\mathcal{L}_{\mathcal{H}}.
$
Therefore,
$
\widehat{\mathcal{R}}_S(\mathcal{L}_{\mathcal{H}_{\mathrm{sym}}})
\le
\widehat{\mathcal{R}}_S(\mathcal{L}_{\mathcal{H}}).
$
A standard Rademacher-complexity bound then implies that, with probability at
least $1-\delta$ over a sample $S=\{(x_i,y_i)\}_{i=1}^n$,
\begin{equation}
\mathcal{E}(h)
\le
\widehat{\mathcal{E}}_S(h)
+
2\,\widehat{\mathcal{R}}_S(\mathcal{L}_{\mathcal{H}})
+
3\sqrt{\frac{\log(2/\delta)}{2n}},
\label{eq:gen_bound_main}
\end{equation}
where $\mathcal{E}(h)$ and $\widehat{\mathcal{E}}_S(h)$ denote the population
and empirical risks, respectively \citep{bartlett2002rademacher}.Applying the same bound to the symmetry-constrained subclass, equivalently to
the induced loss class $\mathcal{L}_{\mathcal{H}_{\mathrm{sym}}}$, yields a
no-larger complexity term. This provides an idealized interpretation of symmetry-consistent
hidden-state augmentation as a form of capacity-reducing regularization. A proof of the monotonicity argument is
provided in Appendix~\ref{app:proof_symmetry_regularization}. This argument is
idealized and is intended to explain the regularizing effect of symmetry
constraints at the level of function classes, rather than as a formal
generalization theorem for the full diffusion training pipeline. This argument establishes monotonicity of the complexity term, but does not
quantify the magnitude of the reduction, which depends on how restrictive the
symmetry constraint is for the underlying hypothesis class.

\begin{figure*}[ht]
    \centering
    \begin{tabular}{cc}
        \includegraphics[width=0.48\linewidth]{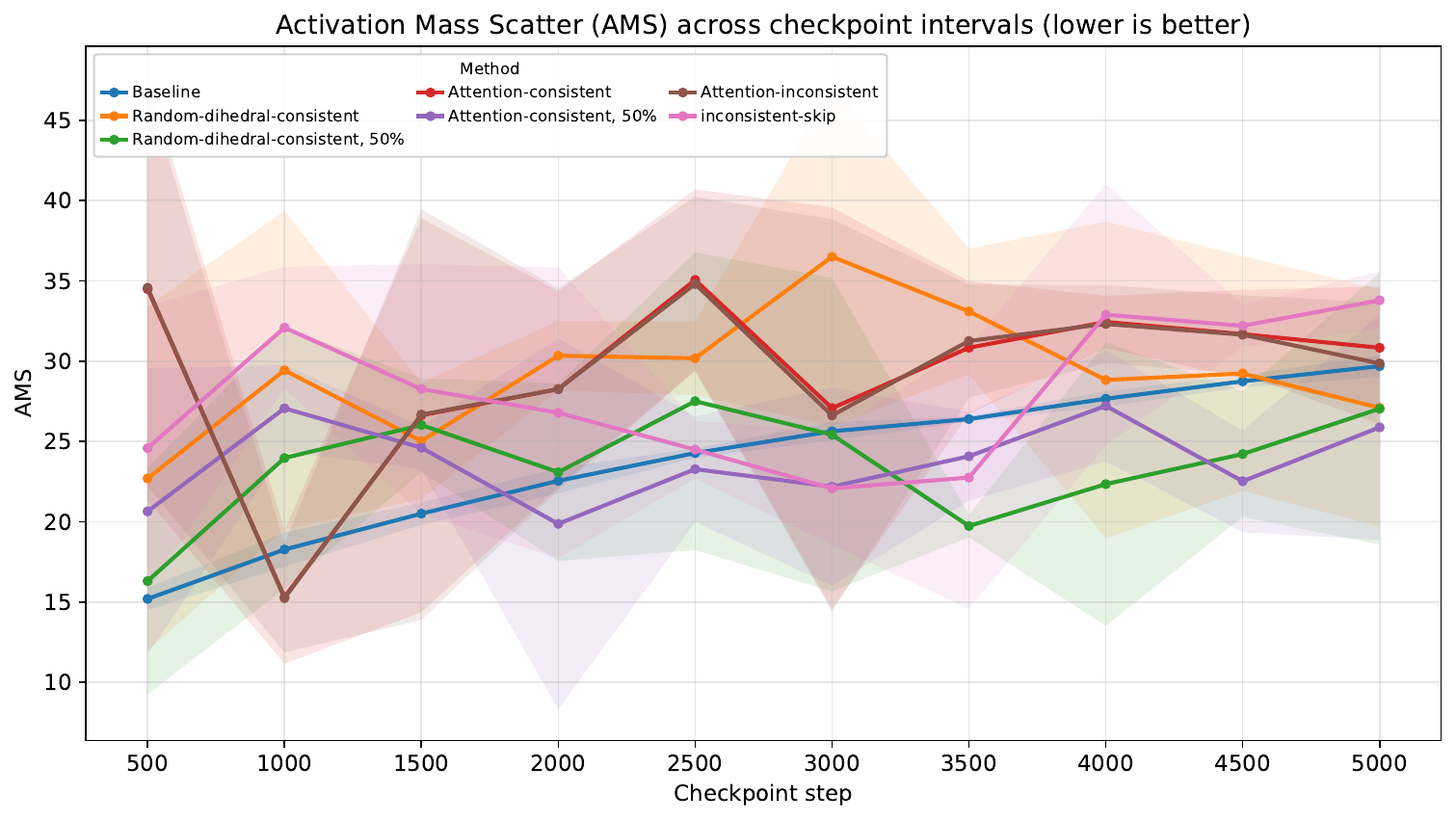} &
        \includegraphics[width=0.48\linewidth]{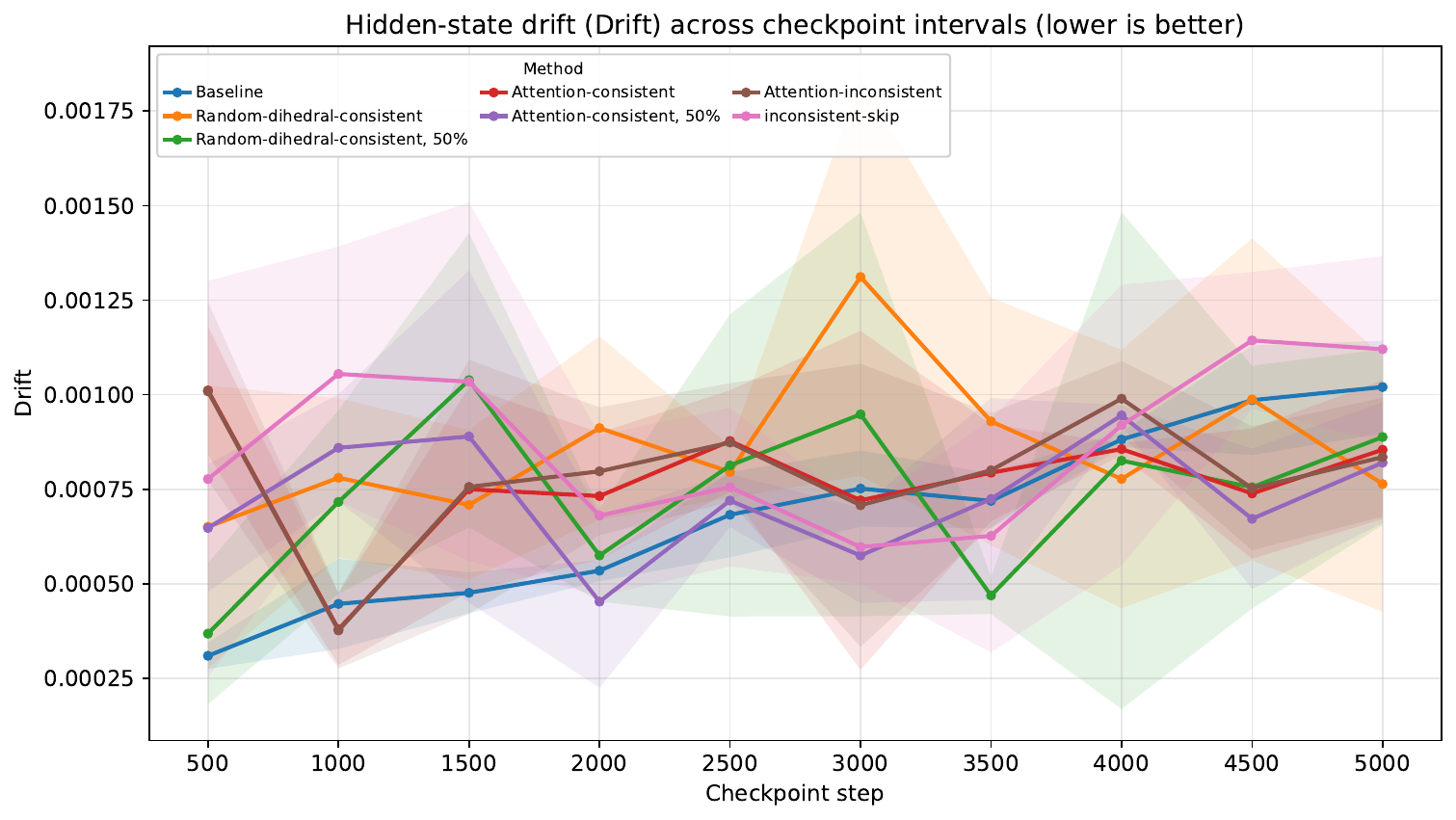} \\
        (a) Activation Mass Scatter (AMS)$\downarrow$ &
        (b) Drift$\downarrow$ \\[8pt]

        \includegraphics[width=0.48\linewidth]{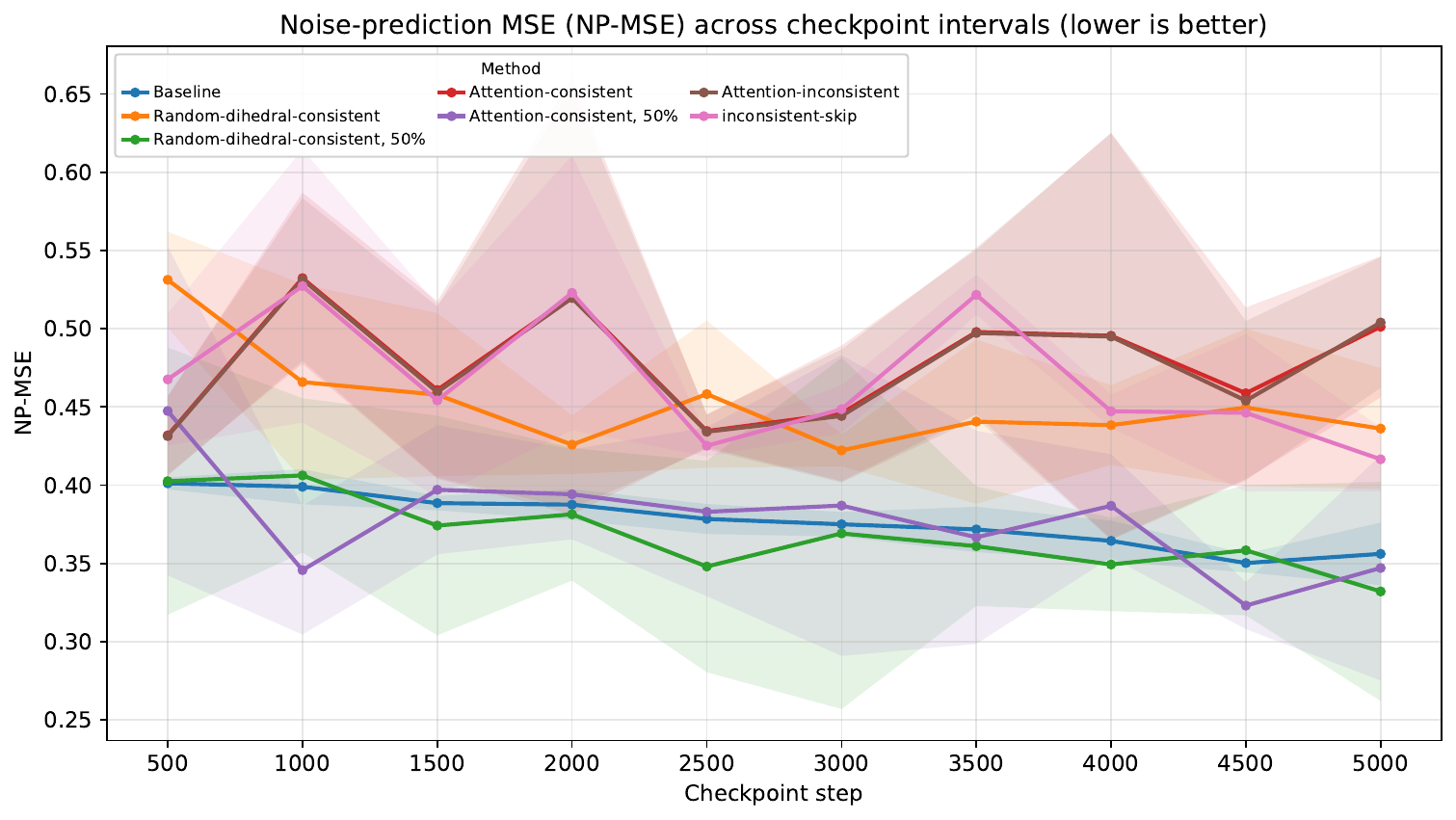} &
        \includegraphics[width=0.48\linewidth]{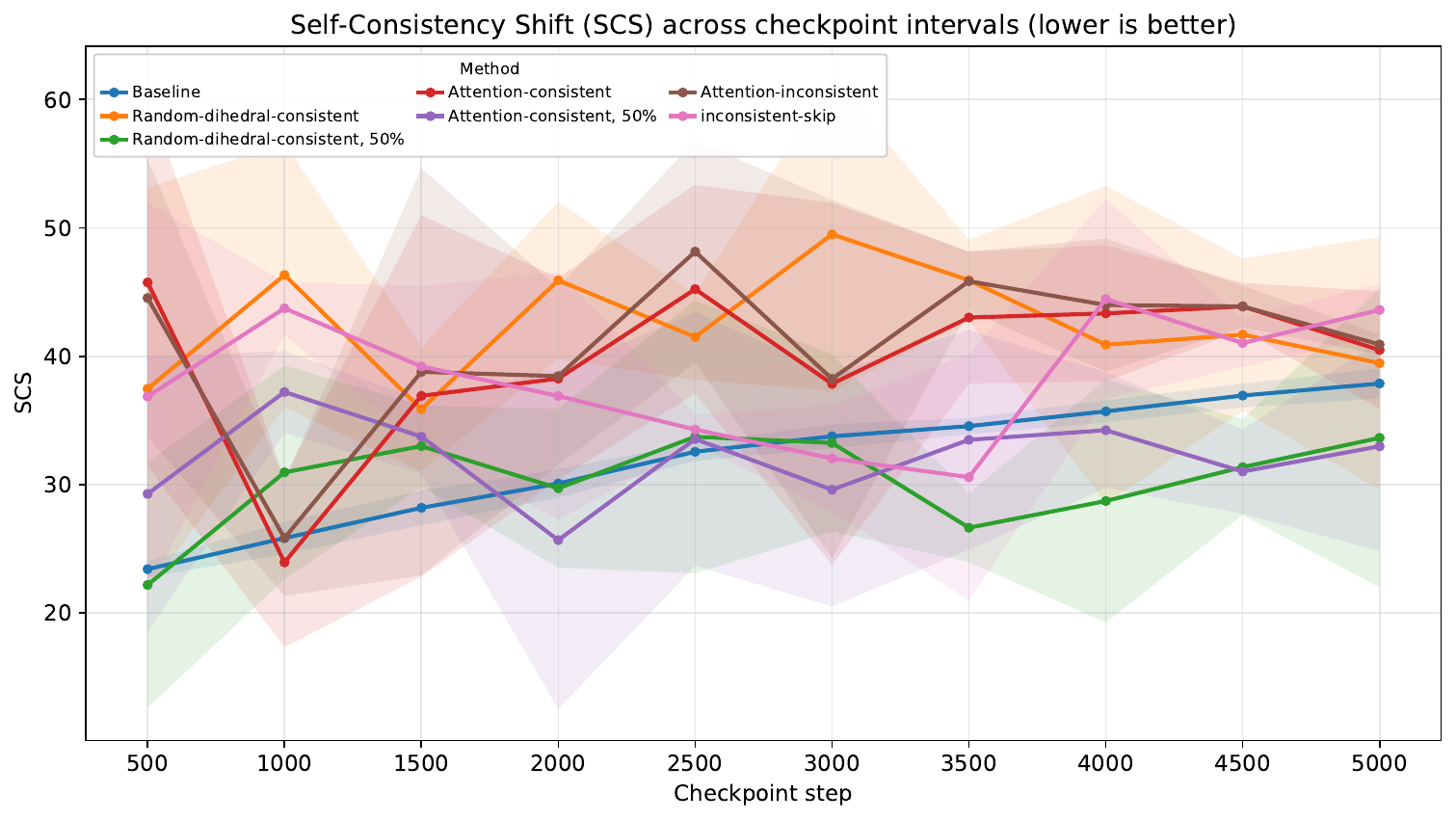} \\
        (c) Noise-Prediction MSE (NP-MSE)$\downarrow$  &
        (d) Self-Consistency Shift (SCS)$\downarrow$
    \end{tabular}
    \caption{
Quantitative evaluation over 5{,}000 steps for seven modes, averaged over three seeds (lower values are better).
}
    \label{fig:combined-metrics}
\end{figure*}
\section{Experiments}
\label{sec:experiments}

We evaluate geometric hidden-state interventions through a main empirical study
on the Stable Diffusion 2.1 U-Net, complemented by additional experiments on
other diffusion backbones, FID-based comparisons, and supporting synthetic and
transformer analyses.

\subsection{Primary Experimental Testbed}
\label{sec:exp_setup}

As one controlled text-to-image testbed, we use the Stable Diffusion 2.1 U-Net,
fine-tuned under a unified intervention framework that applies one sampled
geometric transformation at one sampled location per forward pass. We
train seven modes: \emph{baseline},
\emph{random-dihedral-consistent}, \emph{attention-consistent},
their 50\% variants, and two inconsistent controls,
\emph{attention-inconsistent} and \emph{inconsistent-skip}.
In the 50\% variants, the intervention is applied to a random half of
minibatches; the remaining minibatches use the clean baseline update.
Depending on the mode, interventions are applied at ResNet blocks, attention
blocks, or skip-fusion locations. For controlled experiments, we use 500
randomly sampled Oxford-IIIT Pet images resized to $512\times512$ and
fine-tune for 5{,}000 steps over three seeds. Additional details are deferred
to Appendices~\ref{app:sd_framework_appendix}
and~\ref{app:sd_experiment_detail_appendix}.

We use Oxford-IIIT Pet as a controlled natural-image testbed with a clear
cat/dog split and diverse pose, texture, and background variation. We use the
same random 500-image subset for all modes to keep the seven-mode, three-seed
comparison tractable and controlled.

\subsection{Evaluation Metrics}
\label{sec:exp_metrics}

We evaluate both denoising fidelity and internal geometric stability using four
metrics: \emph{Noise-Prediction Mean Squared Error (NP-MSE)}, \emph{Self-Consistency Shift} (SCS),
\emph{Activation Mass Scatter} (AMS), and \emph{Drift}. The first measures
noise-prediction fidelity, while the latter three quantify complementary
aspects of hidden-state stability under intervention. Lower values are better
for all metrics.

\paragraph{Noise-Prediction MSE (NP-MSE).}
Because full Fréchet Inception Distance (FID) requires large-scale image
generation and an additional Inception-network pass, we use a lightweight
fidelity proxy based on the standard diffusion noise-prediction error:
\begin{equation}
\mbox{$\mathrm{NP\mbox{-}MSE}$}
=
\mathbb{E}_{x_t,t}\big[\|\hat{\varepsilon}_\theta(x_t,t)-\varepsilon\|_2^2\big].
\end{equation}
Here, $x_t$ denotes the noisy latent at diffusion step $t$, $\varepsilon$ the
injected noise, and $\hat{\varepsilon}_\theta(x_t,t)$ the model prediction.
Lower values indicate better denoising performance. Despite the shorthand name,
this metric is simply the denoising noise-prediction error and should not be
interpreted as image-space FID. We use it only as a practical fidelity signal
for our controlled setting.

\paragraph{Self-Consistency Shift (SCS).}
Let $A$ denote the clean feature map and $\widetilde{A}$ the feature map
obtained under a transformation $T$. After aligning the intervened map back to
the clean frame via $\widetilde{A}^{\mathrm{align}}=T^{-1}(\widetilde{A})$, we
define
\begin{equation}
\mathrm{SCS}
=
\frac{1}{BCHW}\|A-\widetilde{A}^{\mathrm{align}}\|_1.
\end{equation}
Lower SCS indicates stronger geometric self-consistency. The benefit of SCS is
that it directly measures whether a symmetry-preserving intervention leaves the
internal representation stable once the intended geometric transformation is
factored out.

\paragraph{Activation Mass Scatter (AMS).}
For a feature map $A\in\mathbb{R}^{B\times C\times H\times W}$, let
$
m_b(h,w)=\sum_{c=1}^{C}|A_{b,c,h,w}|
$
be the spatial activation mass for sample $b$, and let
$
\bar m_b(h,w)=\frac{m_b(h,w)}{\sum_{h,w}m_b(h,w)}
$
be its normalized version. If $\mu_b\in\mathbb{R}^2$ denotes the center of mass
of $\bar m_b$, we define
\begin{equation}
\mathrm{AMS}
=
\frac{1}{B}\sum_{b=1}^{B}\sum_{h,w}\bar m_b(h,w)\,
\bigl\|(h,w)-\mu_b\bigr\|_2^2.
\end{equation}
Lower AMS indicates that activation mass is more spatially concentrated. The
benefit of AMS is that it captures whether an intervention makes the
representation more spatially organized or instead spreads activation mass more
diffusely across the feature map.

\begin{figure*}[tb!]
    \centering
    \begin{tabular}{cc}
        \includegraphics[width=0.48\linewidth]{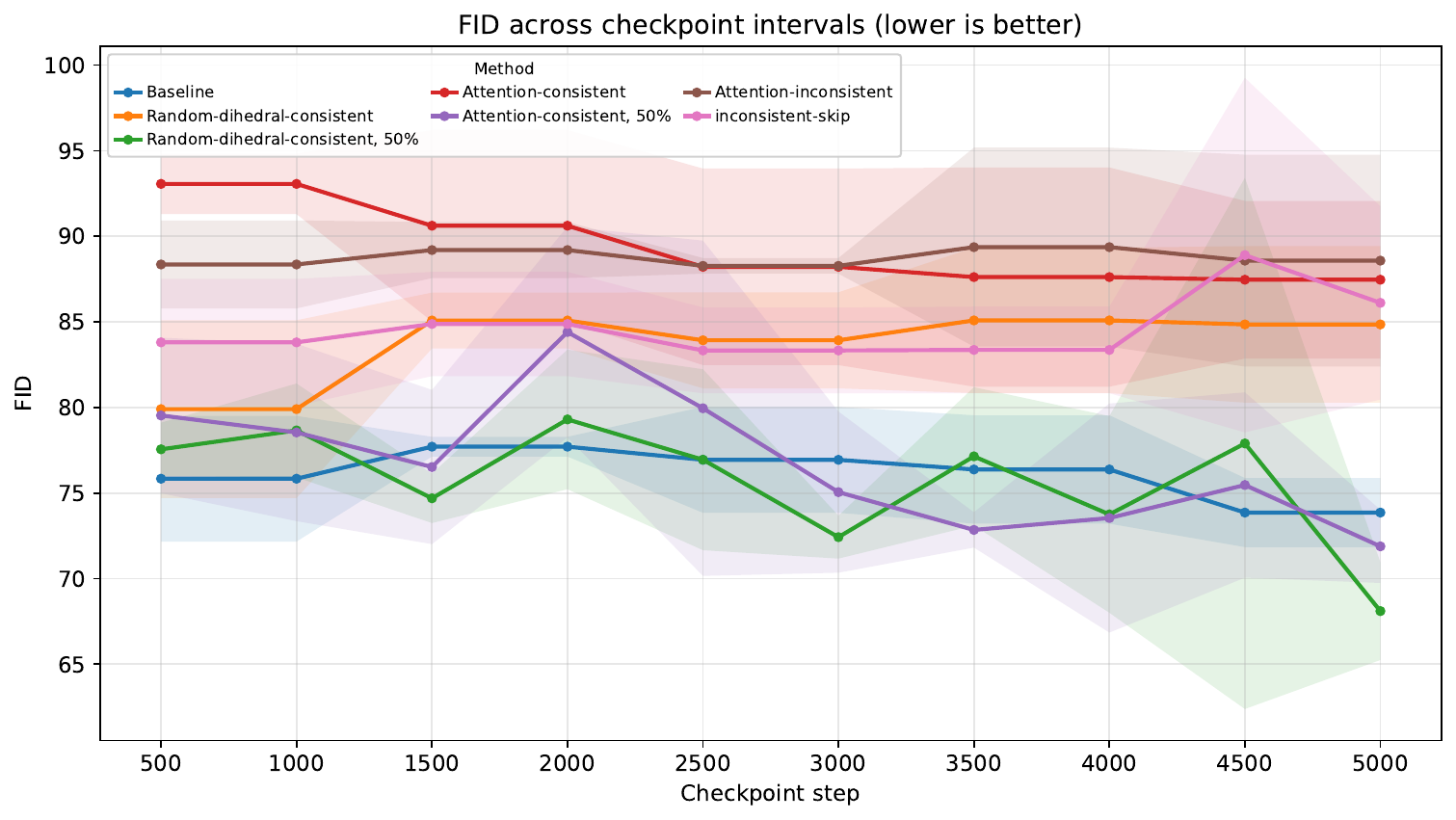} &
        \includegraphics[width=0.48\linewidth]{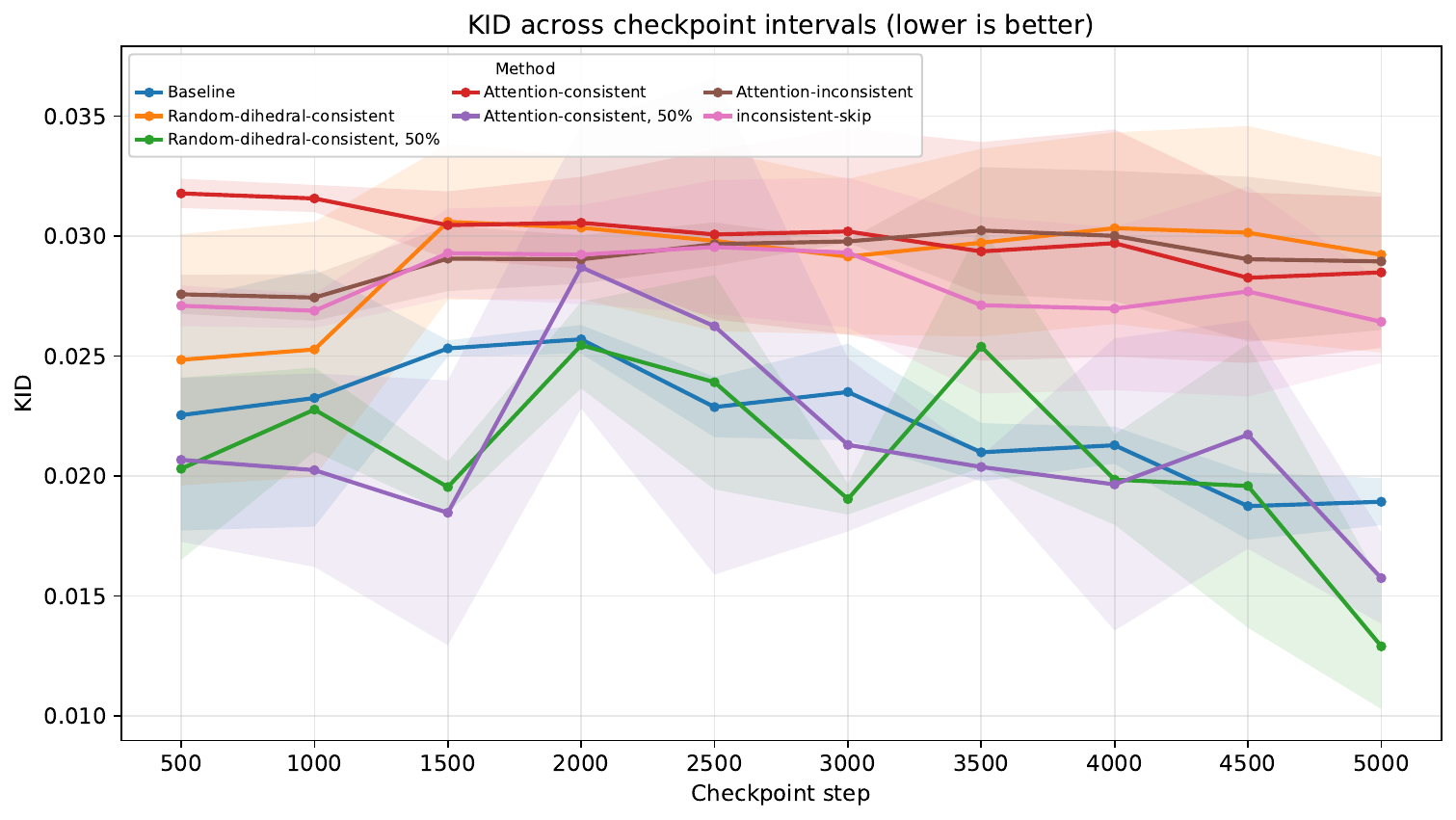} \\
        (a) FID $\downarrow$ &
        (b) KID $\downarrow$ \\[8pt]

        \includegraphics[width=0.48\linewidth]{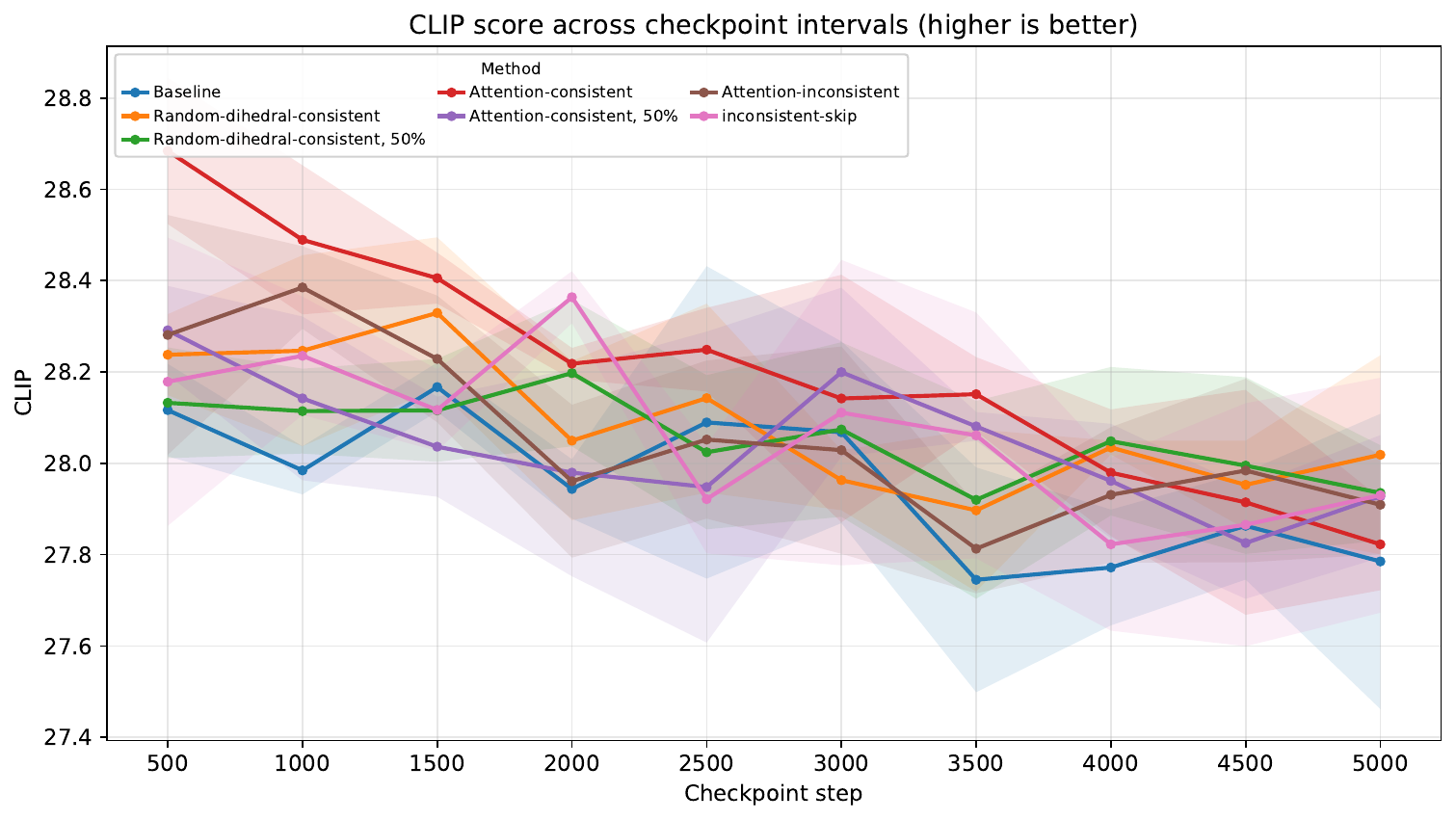} &
        \includegraphics[width=0.48\linewidth]{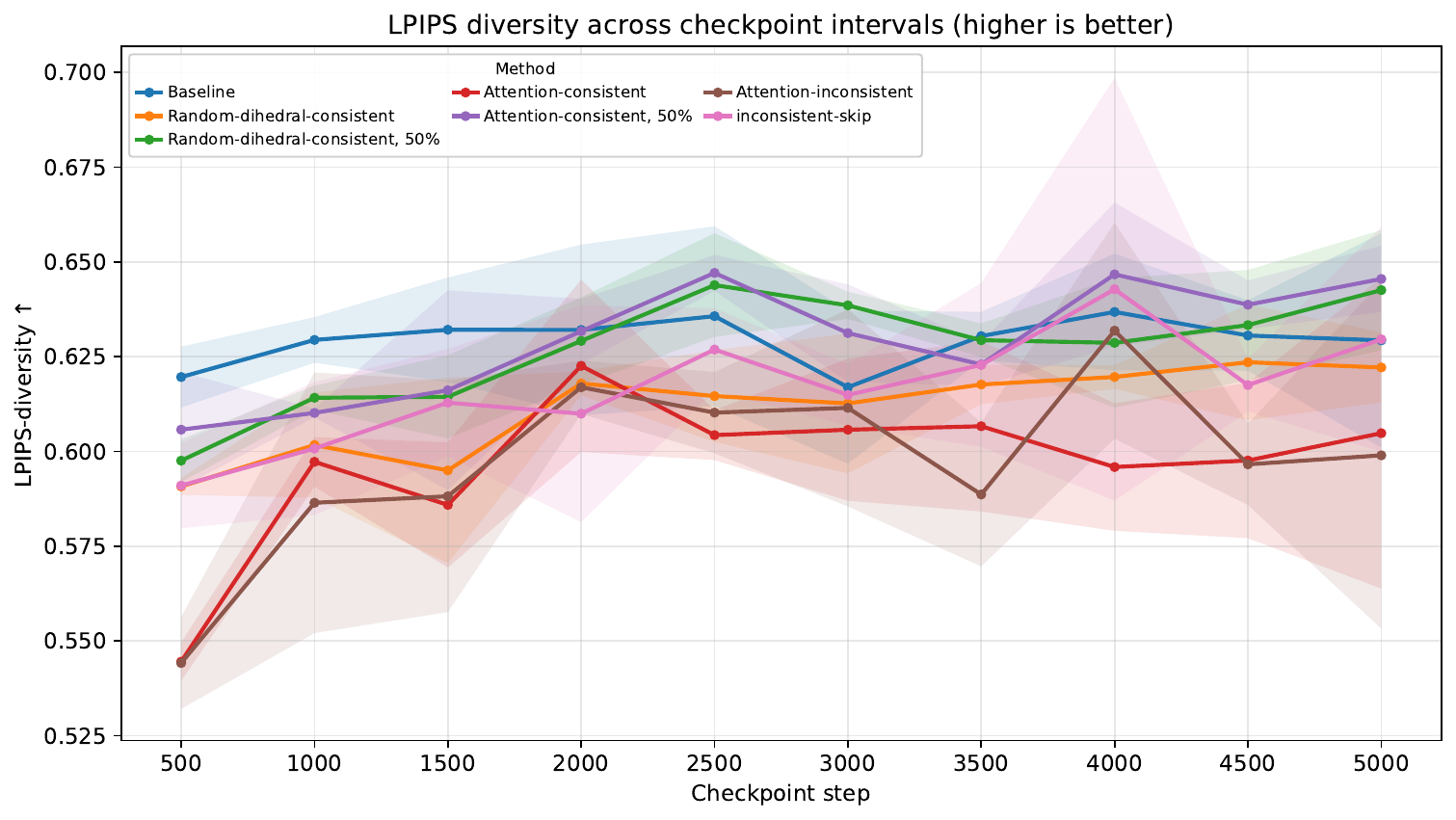} \\
        (c) CLIP score $\uparrow$ &
        (d) LPIPS diversity $\uparrow$
    \end{tabular}
    \caption{
Image-level evaluation over 5{,}000 steps for seven modes and three seeds.
Each mode--seed--checkpoint uses 500 images, totaling 105{,}000 images. Shading shows seed variability.
}
    \label{fig:external-image-metrics}
\end{figure*}

\paragraph{Drift.}
Let $\mu_b$ and $\widetilde{\mu}_b^{\mathrm{align}}$ denote the centers of mass
of the clean and aligned intervened activation maps, respectively. We define
\begin{equation}
\mathrm{Drift}
=
\frac{1}{B}\sum_{b=1}^{B}
\bigl\|\mu_b-\widetilde{\mu}_b^{\mathrm{align}}\bigr\|_2^2.
\end{equation}
Lower Drift indicates stronger global spatial stability under intervention.
Because it captures coarse displacement even when feature content remains
similar, Drift complements SCS. Together, these metrics separate denoising
fidelity from internal geometric stability: NP-MSE measures functional
performance, whereas SCS, AMS, and Drift measure feature consistency, spatial
spread, and global displacement. Additional details are provided in
Appendix~\ref{app:metric_defs_appendix}.

\subsection{Main Results on Stable Diffusion U-Net}
\label{sec:main_results}
We evaluate all seven modes for 5{,}000 steps over three seeds to test whether
hidden-state transformations preserve internal spatial stability, a behavior
not isolated by input-level augmentation alone. Figure~\ref{fig:combined-metrics}
reports NP-MSE, SCS, AMS, and Drift at 500-step intervals; lower is better for
all four metrics. The 50\% consistent variants give the best
stability--fidelity trade-off, reducing geometry-sensitive metrics relative to
full interventions and inconsistent controls while keeping NP-MSE competitive.
We further report image-level FID, KID, CLIP score, and LPIPS diversity at the
same checkpoints (Figure~\ref{fig:external-image-metrics}). For each mode,
seed, and checkpoint, we generate 500 images, 250 cat-prompt and 250 dog-prompt,
for 105{,}000 images total. The 50\% consistent variants remain competitive on
FID/KID while preserving comparable CLIP and LPIPS diversity, indicating that
the internal stability gains do not degrade image-level quality or diversity.
Random qualitative samples, selected uniformly from the full image pool, are
provided in Appendix~\ref{app:qualitative_samples}.

\subsection{Proof-of-Concept FID Comparisons on CIFAR-10, CelebA-64, and MNIST}
\label{sec:baseline_cifar_celeba_mnist}

To complement the main Stable Diffusion 2.1 study, we report small-scale
proof-of-concept FID comparisons on CIFAR-10, CelebA-64, and MNIST. For CIFAR-10 and CelebA-64, we
use the reproduced DDPM baselines of \citep{proszewska2025dmz}; for MNIST, we
use the SDiT reference of \citep{yang2024sdit}. In all regularized variants, we
keep the baseline architecture, diffusion objective, optimizer, sampling
procedure, and evaluation pipeline fixed, and add only our regularization term.

In preliminary tuning, we found that applying the regularizer too rarely left
training too close to the baseline, whereas applying it too frequently degraded
performance. We therefore apply the regularization to a randomly selected half
of the mini-batches during training. Table~\ref{tab:baseline_vs_regularized}
reports the baseline references together with our measured regularized results
under this 50\% schedule. Additional dataset, training, and implementation
details are provided in Appendix~\ref{app:proof_of_concept_details}. These
results should be interpreted as proof-of-concept supporting comparisons under
limited computational resources, rather than as definitive cross-dataset
benchmark claims.

\begin{table}[h]
\centering
\caption{
Baseline references and measured regularized results. Lower FID is better.
Regularized results are mean $\pm$ std. over three runs for CIFAR-10 and
MNIST, and two runs for CelebA-64.
}
\label{tab:baseline_vs_regularized}
\small
\begin{tabular}{llccc}
\toprule
Dataset & Model & Backbone & Resolution & FID $\downarrow$ \\
\midrule
CIFAR-10  & Vanilla \citep{proszewska2025dmz} & DDPM & $32\times32$ & $4.46$  \\
CIFAR-10  & Ours                               & DDPM & $32\times32$ & $\mathbf{4.30 \pm 0.05}$ \\
\midrule
CelebA-64 & Vanilla \citep{proszewska2025dmz} & DDPM & $64\times64$ & $4.51$ \\
CelebA-64 & Ours                               & DDPM & $64\times64$ & $\mathbf{4.53 \pm 0.33}$ \\
\midrule
MNIST     & Vanilla \citep{yang2024sdit}      & SDiT & $28\times28$ & $5.54$ \\
MNIST     & Ours                               & SDiT & $28\times28$ & $\mathbf{5.29 \pm 0.08}$ \\
\bottomrule
\end{tabular}
\end{table}

\subsection{Additional Studies}
\label{sec:additional_studies}

We complement the main U-Net study with supporting transformer-based analyses.
The ViT experiments test geometric consistency in a simpler single-pass
setting, while the DiT study illustrates why the same issue becomes more
consequential in iterative denoisers.
\begin{table*}[t]
\centering
\small

\begin{minipage}[t]{0.44\textwidth}
\centering
\captionof{table}{ViT-B/16 classification accuracy on CIFAR-100. Higher is better.}
\label{tab:vit_cls_main}
\begin{tabular}{lc}
\toprule
Setting & Acc. (\%) \\
\midrule
Input aug. ($224\times224$) & 0.9230 \\
Input aug. ($384\times384$) & 0.9228 \\
Input/hidden aug. ($224\times224$) & 0.9319 \\
\bottomrule
\end{tabular}
\end{minipage}
\hfill
\begin{minipage}[t]{0.53\textwidth}
\centering
\captionof{table}{Synthetic probe consistency scores $S^{(5)}$.}
\label{tab:synthetic_probe_scores}
\begin{tabular}{lcc}
\toprule
Probe & $4\times4$ & $8\times8$ \\
\midrule
H. stripe   & 0.570 & 0.548 \\
V. stripe   & 0.581 & 0.523 \\
Diag. grad. & 0.624 & 0.650 \\
Gaussian    & 1.000 & 1.000 \\
\bottomrule
\end{tabular}
\end{minipage}

\end{table*}

\begin{figure*}[t]
\centering
\setlength{\tabcolsep}{2pt}
\begin{tabular}{cccccccc}
\includegraphics[width=0.11\textwidth]{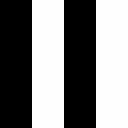} &
\includegraphics[width=0.11\textwidth]{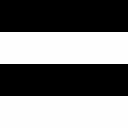} &
\includegraphics[width=0.11\textwidth]{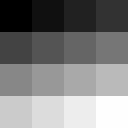} &
\includegraphics[width=0.11\textwidth]{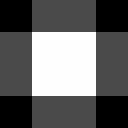} &
\includegraphics[width=0.11\textwidth]{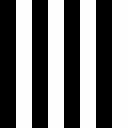} &
\includegraphics[width=0.11\textwidth]{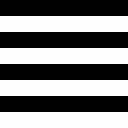} &
\includegraphics[width=0.11\textwidth]{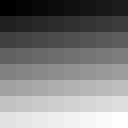} &
\includegraphics[width=0.11\textwidth]{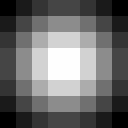} \\

{\scriptsize (a) H-stripe} &
{\scriptsize (b) V-stripe} &
{\scriptsize (c) Diag. grad.} &
{\scriptsize (d) Gaussian} &
{\scriptsize (e) H-stripe} &
{\scriptsize (f) V-stripe} &
{\scriptsize (g) Diag. grad.} &
{\scriptsize (h) Gaussian} \\

{\scriptsize $4\times4,\ D_1$} &
{\scriptsize $4\times4,\ D_1$} &
{\scriptsize $4\times4,\ C_2$} &
{\scriptsize $4\times4,\ D_4$} &
{\scriptsize $8\times8,\ D_1$} &
{\scriptsize $8\times8,\ D_1$} &
{\scriptsize $8\times8,\ C_2$} &
{\scriptsize $8\times8,\ D_4$}
\end{tabular}
\caption{
\textbf{Synthetic symmetry probes with $D_4$ subgroup annotations.}
Controlled probes for measuring geometric sensitivity and anisotropy in
ViT attention under horizontal reflection.
}
\label{fig:synthetic-probes}
\end{figure*}
\subsubsection{ViT Hidden-State Augmentation Improves Generalization}
\label{sec:vit_results_main}

We study a pretrained ViT-B/16 \citep{dosovitskiy2020vit} on CIFAR-100
\citep{krizhevsky2009learning} as a simple transformer testbed. We fine-tune it
for five epochs under three settings: standard input augmentation at
$224\times224$, input augmentation at $384\times384$, and our
input-or-hidden (I/H) augmentation strategy, which applies reflection-based
transformations either to the input or to intermediate hidden states. As shown
in Table~\ref{tab:vit_cls_main}, I/H performs best, improving accuracy from
$92.30\%$ and $92.28\%$ to $93.19\%$. This suggests that symmetry-aware
hidden-state augmentation can already act as an effective regularizer in a
standard ViT setting.

\begin{figure}[t]
    \centering
    \includegraphics[width=0.9\linewidth]{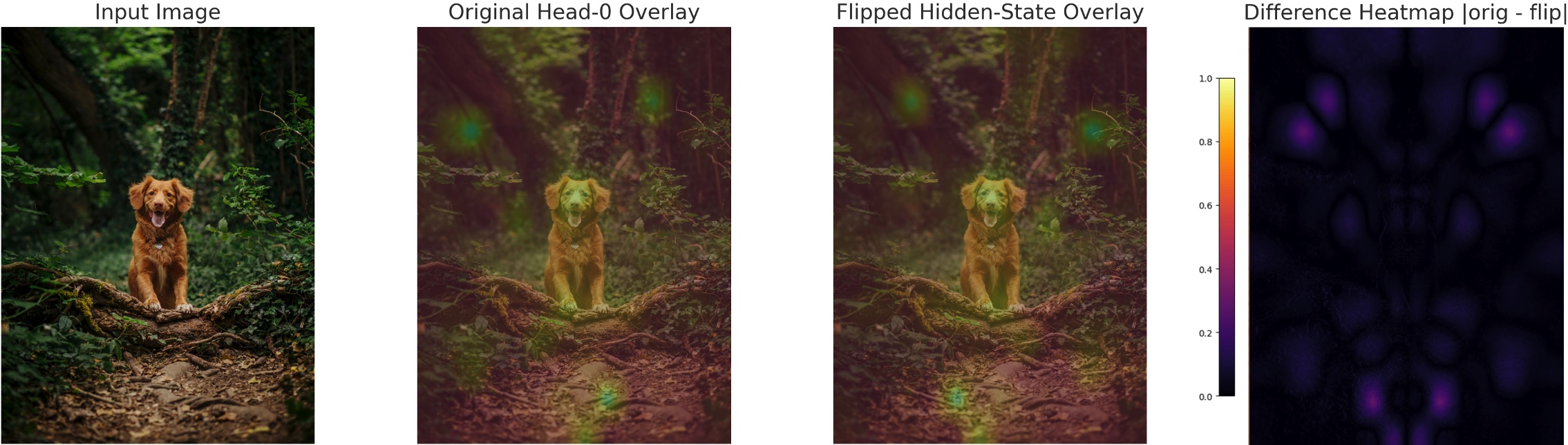}
    \caption{
        \textbf{Hidden-state flipping in ViT-B/16.}
        From left to right: input image, original Block~5 Head~0 attention,
        attention after flipping the Block~5 patch-token hidden state, and the
        difference map $|A-\tilde{A}|$. The intervention preserves semantic
        focus while mirroring spatial routing.
    }
    \label{fig:flip-vis}
\end{figure}

\subsubsection{Synthetic and Real-Image Attention Analysis}
\label{sec:vit_analysis_main}

To examine the geometric effect of hidden-state flipping more directly, we
consider both controlled synthetic probes and real-image attention
visualizations. The synthetic probes isolate geometric structure, while the
real-image study tests whether semantic behavior is preserved.

\paragraph{Synthetic probe evaluation.}
We construct synthetic inputs with explicit geometric structure, including
horizontal and vertical stripes, diagonal gradients, and Gaussian-like blobs on
$4\times4$ and $8\times8$ grids. Using a pretrained ViT-B/16, we compare the
attention response of each probe to that of its horizontally flipped
counterpart via
\begin{equation}
S^{(\ell)}
=
\frac{1}{M}\sum_{i=1}^{M}
\mathrm{cosine}\!\left(A^{(\ell)}(x_i),A^{(\ell)}(T(x_i))\right),
\end{equation}
where $A^{(\ell)}$ denotes the Block~5, Head~0 attention field and $T$ is
horizontal reflection. As shown in Table~\ref{tab:synthetic_probe_scores},
stripe and gradient patterns achieve only moderate consistency, whereas
Gaussian probes achieve perfect consistency. Together with
Fig.~\ref{fig:synthetic-probes}, these results suggest that the selected ViT
attention field responds more consistently to isotropic patterns than to
directional ones under horizontal reflection, revealing measurable anisotropy.

\paragraph{Real-image attention perturbation.}
We next horizontally flip the Block~5 patch-token representation of a
pretrained ViT-B/16 and propagate the modified state through the remaining
encoder blocks. Figure~\ref{fig:flip-vis} shows that the resulting attention
remains focused on the same semantic object parts, but shifts to the mirrored
location. Thus, for this selected attention field, the intervention alters
spatial routing without disrupting semantic focus. In all qualitative
visualizations, we report Head~0 of the selected transformer block; the
rationale is provided in Appendix~\ref{app:head0_appendix}.

\subsubsection{Flipped-Head Attention in DiT: QKV Consistency vs.\ Mismatch}
\label{sec:dit_main}

To connect the ViT analysis to Diffusion Transformers (DiTs), we perform a
controlled synthetic experiment that isolates the geometric effect of
hidden-state flipping inside attention.
This is a controlled mechanistic analysis rather than a full DiT benchmark; it
illustrates how geometric mismatch can become problematic when attention is
applied repeatedly along the denoising trajectory.

We compare the two attention interventions defined earlier: the
\emph{attention-inconsistent} output-only flip of
Eq.~\ref{eq:attn_inconsistent_main}, and the
\emph{attention-consistent} reference corresponding to
Eq.~\ref{eq:attn_consistent_main}, implemented via coherent transformation of
the interacting attention pathways. We consider both a symmetric and a
non-symmetric synthetic baseline: the former tests geometric stability, while
the latter tests geometric correctness under reflection.
Figure~\ref{fig:two_panel_figure} shows the same qualitative conclusion in both
settings. Under the attention-consistent transformation, the attention field
either preserves the baseline geometry (symmetric case) or yields the correct
mirrored response (non-symmetric case). By contrast, the output-only flip
produces shifted and distorted responses, because routing is computed in one
coordinate frame while the resulting head output is interpreted in another. The
distortion heatmap $|A_{\mathrm{incon}}-A_{\mathrm{con}}|$ makes this
deviation explicit.
These results illustrate the geometric-consistency principle: inside attention,
a hidden-state transformation is stable only when all interacting pathways
remain aligned in a common spatial frame. A minimal pseudocode implementation
is provided in Appendix~\ref{app:dit_pseudocode_appendix}.

\begin{figure}[t]
    \centering
    \begin{subfigure}[t]{0.48\linewidth}
        \centering
        \includegraphics[width=\linewidth]{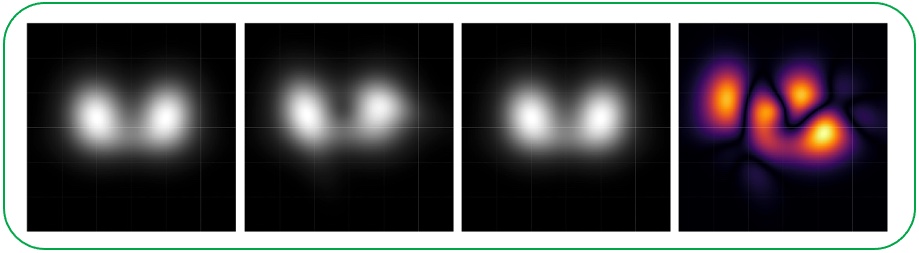}
        \caption{\textbf{Symmetric baseline.} A geometrically valid transformation should remain close to the reference pattern.}
        \label{fig:image1}
    \end{subfigure}
    \hfill
    \begin{subfigure}[t]{0.48\linewidth}
        \centering
        \includegraphics[width=\linewidth]{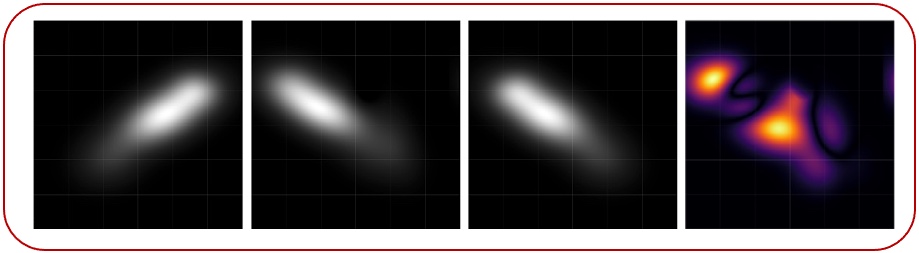}
        \caption{\textbf{Non-symmetric baseline.} A geometrically valid transformation should produce the correct mirrored response.}
        \label{fig:image2}
    \end{subfigure}
    \caption{
\textbf{QKV consistency in flipped-head attention.}
Symmetric and non-symmetric baselines isolate the geometric effect of
hidden-state flipping. In each group, the panels show the baseline,
the attention-inconsistent output-only flip, the attention-consistent QKV
transformation, and the distortion heatmap
$|A_{\mathrm{incon}}-A_{\mathrm{con}}|$. The consistent transformation
preserves the expected geometry, whereas the output-only flip introduces
visible distortion.
}
    \label{fig:two_panel_figure}
\end{figure}

\section{Conclusion}
\label{sec:conclusion_broader_impact}

We introduced a unified framework for \emph{dihedral hidden-state interventions}
and showed that \emph{geometric consistency} determines whether such
interventions stabilize or disrupt computation. Our main contributions are a
common intervention formulation across convolutional and transformer-based
architectures, a theoretical analysis of consistency in attention and U-Net
fusion, and a geometry-aware evaluation suite based on SCS, AMS, Drift, and
NP-MSE. Empirically, in our controlled Stable Diffusion 2.1 study,
appropriately weighted consistent interventions improve hidden-state stability
while preserving image-level quality. Supporting ViT and controlled DiT analyses
show the same qualitative consistency principle. More broadly, these results
suggest that geometric consistency is a useful design principle for
symmetry-aware hidden-state adaptation, while leaving larger-scale studies and
broader transformation sets to future work. Improved internal consistency,
however, does not by itself address broader concerns such as bias, fairness, or
misuse.

\section*{Acknowledgments}
This work was supported in part by funding from the Canada Excellence Research Chairs (CERC) program. The authors gratefully acknowledge this support. The authors also acknowledge Compute Canada and Calcul Québec for providing the computing resources used in this work.

\bibliography{main}
\bibliographystyle{collas2026_conference}

\appendix
\appendix

\section*{Appendix}
This appendix provides supplementary derivations, proofs, implementation
details, and extended discussions supporting the main paper. We group the
material into:
(i) formal derivations for flipped-head interventions,
(ii) extended consistency analysis for transformer denoisers,
(iii) supplementary details for the Stable Diffusion UNet framework and
experiments, and
(iv) detailed definitions of the evaluation metrics.

\section{Detailed Derivation of Flipped-Head Attention in ViT}
\label{app:vit_flip_derivation}

This appendix gives the index-level form of the flipped-head intervention used
in the main text.

Let $h^\star$ be the selected attention head, and let
\[
O_{h^\star}\in\mathbb{R}^{N\times d_v},\qquad N=P^2.
\]
Define the token-to-grid mapping
\[
t(i,j)=(i-1)P+j,
\]
and reshape the head output into
\[
\mathcal{O}_{h^\star}\in\mathbb{R}^{P\times P\times d_v},
\qquad
\mathcal{O}_{h^\star}[i,j,:]=O_{h^\star}[t(i,j),:].
\]

For a horizontal flip operator $\mathcal{F}$,
\[
(\mathcal{F}\mathcal{O}_{h^\star})[i,j,:]
=
\mathcal{O}_{h^\star}[i,P-j+1,:].
\]
Flattening back yields
\[
\widetilde{O}_{h^\star}[t(i,j),:]
=
(\mathcal{F}\mathcal{O}_{h^\star})[i,j,:].
\]

Replacing only head $h^\star$ gives
\[
\widehat{O}
=
\mathrm{Concat}(O_1,\dots,O_{h^\star-1},\widetilde{O}_{h^\star},
O_{h^\star+1},\dots,O_H),
\]
and the resulting block output is
\[
Y=\widehat{O}W^O.
\]
Equivalently, for token $(i,j)$,
\[
\mathbf{y}_{t(i,j)}
=
W^O
\begin{bmatrix}
O_1[t(i,j),:] \\
\vdots\\
O_{h^\star-1}[t(i,j),:]\\
O_{h^\star}[t(i,P-j+1),:]\\
O_{h^\star+1}[t(i,j),:]\\
\vdots\\
O_H[t(i,j),:]
\end{bmatrix}.
\]

Thus, the intervention leaves the attention weights unchanged and modifies only
the post-attention representation of one head, so each token mixes standard
contextual features with a mirrored contextual feature from the transformed
head.

\section{Extended DiT Consistency Analysis}
\label{app:qkv_appendix}

This appendix provides additional justification for the consistency analysis
introduced in Sec.~\ref{sec:method}.

\subsection{Permutation Equivariance of a Single Attention Head}

Let
\[
\mathcal{A}(Q,K,V)
=
\mathrm{softmax}\!\left(\frac{QK^\top}{\sqrt{d_k}}\right)V
\]
and let $\Pi_T$ be the permutation matrix induced by a spatial transformation
$T$ on token indices. Then
\[
\mathcal{A}(\Pi_TQ,\Pi_TK,\Pi_TV)=\Pi_T\,\mathcal{A}(Q,K,V).
\]

\paragraph{Proof.}
Let
\[
Q'=\Pi_TQ,\qquad K'=\Pi_TK,\qquad V'=\Pi_TV.
\]
Then
\[
Q'K'^\top = \Pi_TQK^\top \Pi_T^\top.
\]
Since softmax is row-wise and commutes with simultaneous row/column
permutations,
\[
\mathrm{softmax}(\Pi_TQK^\top \Pi_T^\top)
=
\Pi_T\,\mathrm{softmax}(QK^\top)\,\Pi_T^\top.
\]
Therefore
\[
\mathcal{A}(Q',K',V')
=
\mathrm{softmax}(Q'K'^\top)V'
=
\Pi_T\,\mathrm{softmax}(QK^\top)V
=
\Pi_T\,\mathcal{A}(Q,K,V).
\]

This shows that equivariance holds for an isolated transformed head. The
inconsistency discussed in the main text therefore does not arise from the
single-head attention map itself, but from mixing transformed and untransformed
heads within the same multi-head block.

\subsection{Heuristic Error Propagation Under Repeated Attention Perturbations}

We now give a heuristic linearized argument for why repeated denoising can
amplify the effect of geometric inconsistency in transformer diffusion models.

Let
\[
\widetilde{O}_t = O_t + E_t,
\]
where $E_t$ denotes the perturbation induced by the inconsistent intervention at
step $t$. Consider a schematic denoising update
\[
z_{t-1} = \Psi_t(z_t,\hat{\epsilon}_t),
\]
where $\hat{\epsilon}_t$ is the predicted noise and $\Psi_t$ is the one-step
sampler update. Linearizing with respect to the perturbation gives
\[
\hat{\epsilon}_t^{\mathrm{inc}}
=
\hat{\epsilon}_t + J_tE_t,
\]
where $J_t=\partial\hat{\epsilon}_t/\partial O_t$ is the local Jacobian of the
noise predictor with respect to the attention output. The resulting deviation in
the denoising trajectory is approximately
\[
\delta z_{t-1}\approx D_tJ_tE_t,
\]
where $D_t$ is the local sensitivity of the sampler update to perturbations in
the predicted noise. Unrolling across denoising steps yields the heuristic
accumulation
\[
\Delta_T \approx \sum_{t=1}^{T} D_tJ_tE_t.
\]
Thus, even small per-step geometric inconsistencies can accumulate into a
nontrivial trajectory deviation over many denoising steps.

\subsection{Diffusion Timestep Sensitivity: ViT vs.\ DiT}

For a ViT, attention is applied in a single forward pass. For a DiT, attention
is applied repeatedly across denoising steps. Under the heuristic accumulation
above, the relative sensitivity may scale with the number of denoising steps as
\[
R
=
\frac{\|\Delta^{\mathrm{DiT}}\|}{\|\Delta^{\mathrm{ViT}}\|}
\approx
\frac{\sum_{t=1}^{T}\|D_tJ_tE_t\|}{\|E_1\|}.
\]
This suggests that repeated denoising can amplify the effect of geometric
inconsistency more strongly in DiT than in a single-pass ViT.

\section{Additional Clarification on Diffusion Hidden-State Interventions}
\label{app:diffusion_details}

The hidden-state intervention objective introduced in
Sec.~\ref{sec:method} applies unchanged to both U-Net and DiT diffusion
backbones. For a sampled block $\ell$ and transformation $\tau\in\mathcal{T}$, we
replace the original hidden activation
\[
F^{(\ell)}(x_t,t)
\]
by its transformed version
\[
\widetilde{F}^{(\ell)}(x_t,t)=\tau(F^{(\ell)}(x_t,t)),
\]
and denote by $\varepsilon_\theta^{(\tau,\ell)}(x_t,t)$ the denoiser output of
the modified network. The training objective then penalizes deviation between
the target noise $\varepsilon$ and the prediction
$\varepsilon_\theta^{(\tau,\ell)}(x_t,t)$, as defined in
Eq.~\ref{eq:diff_aug_loss_main}.

What differs across architectures is not the objective itself, but the form of
the coupled computation that must remain geometrically aligned. In U-Nets, the
critical interaction arises at skip-connected encoder--decoder fusion: paired
feature maps
$F_{\mathrm{enc}}^{(\ell)}$ and $F_{\mathrm{dec}}^{(\ell')}$ at matched spatial
resolution are combined through a fusion operator $\phi$, so geometric
consistency requires that both branches be expressed in the same transformed
frame before fusion. In DiT-style denoisers, the corresponding interaction
arises through attention and token mixing, as analyzed in
Appendix~\ref{app:qkv_appendix}; there, consistency requires the interacting
attention pathways to share a common transformed frame.

In our implementation, one intervention location $\ell$ and one transformation
$\tau$ are sampled per mini-batch.

\section{Proofs for Sec.~\ref{sec:theory}}
\label{app:theory_proofs}

\subsection{Proof of Lemma~\ref{lem:flip_correspondence_main}}
\label{app:proof_flip_correspondence}

By construction, the reshaped head output satisfies
\[
\mathcal{O}_{h^\star}[i,j,:]=O_{h^\star}[t(i,j),:].
\]
A horizontal flip acts as
\[
(T\mathcal{O}_{h^\star})[i,j,:]
=
\mathcal{O}_{h^\star}[i,P-j+1,:].
\]
Flattening back defines
\[
\widetilde{O}_{h^\star}[t(i,j),:]
=
(T\mathcal{O}_{h^\star})[i,j,:].
\]
Substituting the flipped expression gives
\[
\widetilde{O}_{h^\star}[t(i,j),:]
=
\mathcal{O}_{h^\star}[i,P-j+1,:]
=
O_{h^\star}[t(i,P-j+1),:],
\]
which is exactly the claimed correspondence.

\subsection{Proof of Proposition~\ref{prop:final_token_main}}
\label{app:proof_final_token}

Let
\[
\widehat{O}
=
\mathrm{Concat}(O_1,\dots,O_{h^\star-1},\widetilde{O}_{h^\star},
O_{h^\star+1},\dots,O_H)
\]
denote the intervened multi-head representation. Partition the output
projection as
\[
W^O=[W_1^O\;\cdots\;W_H^O],
\qquad
W_h^O\in\mathbb{R}^{d_{\mathrm{model}}\times d_v}.
\]
Then for any token index $t$,
\[
\mathbf{y}_t
=
\sum_{h=1}^{H} W_h^O\,\widehat{o}_h[t],
\]
where $\widehat{o}_h[t]$ is the token feature contributed by head $h$ after the
intervention. For $h\neq h^\star$, we have
\[
\widehat{o}_h[t(i,j)] = o_h[t(i,j)].
\]
For the selected head, Lemma~\ref{lem:flip_correspondence_main} gives
\[
\widehat{o}_{h^\star}[t(i,j)]
=
o_{h^\star}[t(i,P-j+1)].
\]
Substituting these into the block expansion of $\mathbf{y}_{t(i,j)}$ yields
\[
\mathbf{y}_{t(i,j)}
=
\sum_{h\neq h^\star} W_h^O\,o_h[t(i,j)]
+
W_{h^\star}^O\,o_{h^\star}[t(i,P-j+1)],
\]
which proves the claim.

\subsection{Proof of Proposition~\ref{prop:attn_consistency_main}}
\label{app:proof_attn_consistency}

If every head output is transformed by the same permutation $\Pi_T$, then
\[
Y_{\mathrm{cons}}
=
\mathrm{Concat}(\Pi_T O_1,\dots,\Pi_T O_H)W^O.
\]
Because $\Pi_T$ acts on the token axis while $W^O$ acts on the concatenated
feature axis, the permutation factors out:
\[
\mathrm{Concat}(\Pi_T O_1,\dots,\Pi_T O_H)
=
\Pi_T\,\mathrm{Concat}(O_1,\dots,O_H).
\]
Hence
\[
Y_{\mathrm{cons}}
=
\Pi_T\,\mathrm{Concat}(O_1,\dots,O_H)W^O
=
\Pi_T Y.
\]

Now suppose only a strict subset $S\subsetneq\{1,\dots,H\}$ of heads is
transformed. Then the intervened output has the form
\[
Y_{\mathrm{partial}}
=
\mathrm{Concat}(\widetilde{O}_1,\dots,\widetilde{O}_H)W^O,
\qquad
\widetilde{O}_h=
\begin{cases}
\Pi_T O_h, & h\in S,\\
O_h, & h\notin S.
\end{cases}
\]
We claim that, in general, $Y_{\mathrm{partial}}\neq \Pi_T Y$. To see this, it
suffices to construct a counterexample. Choose some $h_0\notin S$ and let all
other head outputs be zero, while $O_{h_0}$ is nonzero and not invariant under
$\Pi_T$. Then
\[
Y_{\mathrm{partial}} = O_{h_0}W_{h_0}^O,
\qquad
\Pi_T Y = \Pi_T(O_{h_0}W_{h_0}^O).
\]
Since $O_{h_0}$ is not invariant under $\Pi_T$, these two quantities differ in
general. Therefore, selective transformation of only part of the multi-head
module does not preserve a common transformed frame.

\subsection{Proof of Proposition~\ref{prop:skip_consistency_main}}
\label{app:proof_skip_consistency}

By assumption, the fusion operator $\phi$ is equivariant under simultaneous
spatial permutation:
\[
\phi(\Pi_T a,\Pi_T b)=\Pi_T\phi(a,b).
\]
Applying this with
\[
a=F_{\mathrm{dec}}^{(\ell')}
\qquad\text{and}\qquad
b=F_{\mathrm{enc}}^{(\ell)}
\]
gives
\[
\phi(\Pi_T F_{\mathrm{dec}}^{(\ell')},\Pi_T F_{\mathrm{enc}}^{(\ell)})
=
\Pi_T\,\phi(F_{\mathrm{dec}}^{(\ell')},F_{\mathrm{enc}}^{(\ell)}),
\]
which proves the consistency relation.

To see why one-sided transformation is inconsistent in general, consider a
standard fusion operator such as addition:
\[
\phi(a,b)=a+b.
\]
Then
\[
\phi(\Pi_T a,b)=\Pi_T a+b,
\qquad
\Pi_T\phi(a,b)=\Pi_T a+\Pi_T b.
\]
These are equal only in special cases, for example when $b=\Pi_T b$. Thus,
applying the transformation to only one branch generally fails to preserve the
common transformed frame. The same conclusion holds for common skip-fusion
operations that depend nontrivially on both inputs, such as concatenation
followed by a shared convolution.
\subsection{Proof of the symmetry-regularization claim}
\label{app:proof_symmetry_regularization}

Since $\mathcal{H}_{\mathrm{sym}}\subseteq\mathcal{H}$, the supremum over
$\mathcal{H}_{\mathrm{sym}}$ is bounded by the supremum over $\mathcal{H}$.
Therefore, for any sample $S=\{x_i\}_{i=1}^n$,
\[
\widehat{\mathcal{R}}_S(\mathcal{H}_{\mathrm{sym}})
=
\mathbb{E}_{\sigma}
\left[
\sup_{h\in\mathcal{H}_{\mathrm{sym}}}
\frac{1}{n}\sum_{i=1}^n \sigma_i h(x_i)
\right]
\le
\mathbb{E}_{\sigma}
\left[
\sup_{h\in\mathcal{H}}
\frac{1}{n}\sum_{i=1}^n \sigma_i h(x_i)
\right]
=
\widehat{\mathcal{R}}_S(\mathcal{H}).
\]

Next, because $\mathcal{H}_{\mathrm{sym}}\subseteq\mathcal{H}$, the induced
loss classes satisfy
\[
\mathcal{L}_{\mathcal{H}_{\mathrm{sym}}}
\subseteq
\mathcal{L}_{\mathcal{H}}.
\]
Applying the same monotonicity argument yields
\[
\widehat{\mathcal{R}}_S(\mathcal{L}_{\mathcal{H}_{\mathrm{sym}}})
\le
\widehat{\mathcal{R}}_S(\mathcal{L}_{\mathcal{H}}).
\]

A standard Rademacher-complexity generalization bound for bounded losses then
implies that applying the same bound to the symmetry-constrained subclass,
equivalently to its induced loss class, yields a no-larger complexity term. This gives the claimed idealized
regularization interpretation.
The argument therefore shows only that the symmetry-constrained class has a
no-larger Rademacher complexity term; it does not measure how tight or large
the reduction is in a particular model class.

\section{Stable Diffusion UNet Framework}
\label{app:sd_framework_appendix}

Our unified framework inserts exactly one sampled geometric transformation from
the retained reflection subset at one selected location during each forward
pass. Depending on the mode, candidate
locations include ResNet blocks, attention blocks, and skip-connected fusion
points.

\subsection{Intervention Modes}

\begin{itemize}
    \item \textbf{Baseline:} no intervention.
    \item \textbf{Random-dihedral-consistent:} the same sampled
    transformation is applied consistently across interacting pathways.
    \item \textbf{Attention-consistent:} the transformation is applied
    uniformly to $Q$, $K$, and $V$ before attention.
    \item \textbf{Attention-inconsistent:} only the value projection or
    attention output is transformed.
    \item \textbf{Inconsistent-skip:} only one side of a skip connection
    is transformed.
\end{itemize}

\subsection{Algorithm}
\begin{algorithm}[t]
\caption{Schematic intervention framework for Stable Diffusion UNet}
\label{alg:sd_unet_framework}
\begin{algorithmic}[1]
\Require mode $M$, candidate sets $C_{\mathrm{res}}$, $C_{\mathrm{attn}}$, $C_{\mathrm{skip}}$, transform set $\mathcal{T}$
\State
\[
C \leftarrow
\begin{cases}
\varnothing, & M = \texttt{baseline},\\
C_{\mathrm{res}} \cup C_{\mathrm{skip}}, & M = \texttt{random-dihedral-consistent},\\
C_{\mathrm{attn}}, & M = \texttt{attention-consistent}.
\end{cases}
\]
\If{$C \neq \varnothing$}
    \State Sample target block $B^\star \in C$
    \State Sample transformation $\tau \sim \mathcal{T}$
\EndIf
\For{each block $B$ in the UNet}
    \If{$M = \texttt{baseline}$ or $B \neq B^\star$}
        \State $x \leftarrow B(x)$
    \Else
        \State $x \leftarrow \textsc{ApplyIntervention}(B,x,M,\tau)$
    \EndIf
\EndFor
\State \Return output latent $x$
\end{algorithmic}
\end{algorithm}

Algorithm~\ref{alg:sd_unet_framework} summarizes the retained Stable Diffusion
UNet training framework used in our experiments.

In Algorithm~\ref{alg:sd_unet_framework},  $\textsc{ApplyIntervention}(B,x,M,\tau)$ denotes the mode-specific
operation at the sampled location. For
random-dihedral-consistent, the sampled transformation is applied
consistently across the interacting pathways at that location. For
attention-consistent, the transformation is applied coherently inside
attention (e.g.\ to the interacting attention pathways). For
attention-inconsistent, only part of the attention computation is
transformed.

This retained framework supports a controlled comparison of the implemented
geometrically consistent intervention modes under a shared training protocol.

\section{Stable Diffusion UNet Experimental Details}
\label{app:sd_experiment_detail_appendix}

We construct a lightweight but diverse experimental dataset by randomly sampling
500 images from the Oxford-IIIT Pet dataset~\cite{parkhi2012cats}. All images are resized to
$512\times512$ and encoded by the Stable Diffusion 2.1 VAE into latent tensors
of shape $4\times64\times64$.

\subsection{Training Setup}

Each experiment is run for 5,000 training steps. We use:
\begin{itemize}
    \item batch size $4$,
    \item AdamW optimizer with learning rate $10^{-5}$,
    \item a standard DDPM training scheduler,
    \item evaluation every 50 steps,
    \item final evaluation using full-UNet passes without intervention.
\end{itemize}

We keep all optimization and evaluation hyperparameters fixed across modes so
that differences can be attributed to the intervention type rather than
mode-specific tuning. The batch size and learning rate were chosen as a stable
lightweight fine-tuning configuration for Stable Diffusion 2.1 under our GPU
memory constraints. For the 50\% variants, preliminary runs showed that applying
the intervention too rarely stayed close to the baseline, whereas applying it to
every minibatch could be overly strong; we therefore use a randomly selected
50\% of minibatches as a fixed milder schedule.

\subsection{Reflection Transform Set}
For each forward pass, once a target layer is selected, we apply a randomly chosen reflection transformation from
\[
\mathcal{T}
=
\{T_{\mathrm{hor}}, T_{\mathrm{ver}}, T_{\mathrm{diag}}, T_{\mathrm{anti}}\}.
\]
We exclude rotations because VAE latents in Stable Diffusion are not reliably
rotation-consistent.

\subsection{Random Selection of Augmentation Location}

Depending on the mode, augmentation is applied at exactly one randomly selected
location during each forward pass:
\begin{itemize}
    \item \textbf{ResNet modes:} candidates from down, mid, and up blocks.
    \item \textbf{Attention modes:} candidates include attention blocks.
    \item \textbf{Baseline mode:} no augmentation.
\end{itemize}

Sampling follows
\[
B^\star \sim \mathrm{Uniform}(C), \qquad \tau \sim \mathcal{T}.
\]

\subsection{Why Random Layer Selection Matters}

Stable Diffusion is hierarchical and spatially structured. Down blocks capture
global geometry, mid blocks capture coarse semantics, and up blocks inject fine
detail. Random perturbation location therefore:
\begin{itemize}
    \item tests equivariance across multiple scales,
    \item avoids overfitting to a specific layer,
    \item exposes the network to geometric perturbations throughout its depth.
\end{itemize}


\section{Metric Definitions and Implementation Details}
\label{app:metric_defs_appendix}

This appendix specifies how the evaluation metrics in
Sec.~\ref{sec:exp_metrics} are computed in practice.

\paragraph{Noise-Prediction Mean Squared Error (NP-MSE).}
We use
\[
\mbox{$\mathrm{NP\text{-}MSE}$}
=
\mathbb{E}_{x_t,t}\big[\|\hat{\varepsilon}_\theta(x_t,t)-\varepsilon\|_2^2\big]
\]
as a lightweight fidelity proxy. This is simply the standard diffusion
noise-prediction error averaged over evaluation samples and timesteps. We use
\emph{Noise-Prediction Mean Squared Error (NP-MSE)} as the direct name for this quantity;
it should not be interpreted as an estimator of image-space Fréchet Inception
Distance.

\paragraph{Feature alignment.}
For geometry-sensitive metrics, the clean activation is denoted by $A$ and the
intervened activation by $\widetilde{A}$. When the intervention applies a
transformation $T$, we first align the intervened activation back to the clean
frame via
\[
\widetilde{A}^{\mathrm{align}}=T^{-1}(\widetilde{A}).
\]
This ensures that SCS and Drift measure instability beyond the intended
geometric transformation itself.

\paragraph{Self-Consistency Shift (SCS).}
After alignment, SCS is computed as
\[
\mathrm{SCS}
=
\frac{1}{BCHW}\|A-\widetilde{A}^{\mathrm{align}}\|_1.
\]
This metric captures feature-level mismatch between clean and intervened
representations after both are expressed in the same spatial frame.

\paragraph{Activation Mass Scatter (AMS).}
For each sample $b$, we define the spatial activation mass
\[
m_b(h,w)=\sum_{c=1}^{C}|A_{b,c,h,w}|,
\qquad
\bar m_b(h,w)=\frac{m_b(h,w)}{\sum_{h,w}m_b(h,w)}.
\]
Its center of mass is
\[
\mu_b
=
\sum_{h,w}\bar m_b(h,w)\,(h,w).
\]
AMS is then defined as the second spatial moment of the normalized mass:
\[
\mathrm{AMS}
=
\frac{1}{B}\sum_{b=1}^{B}\sum_{h,w}\bar m_b(h,w)\,
\bigl\|(h,w)-\mu_b\bigr\|_2^2.
\]
Thus AMS measures how spatially dispersed the activation mass is around its own
center.

\paragraph{Drift.}
Let $\mu_b$ be the center of mass of the clean activation map and let
$\widetilde{\mu}_b^{\mathrm{align}}$ be the center of mass of the aligned
intervened activation map. Drift is defined by
\[
\mathrm{Drift}
=
\frac{1}{B}\sum_{b=1}^{B}
\bigl\|\mu_b-\widetilde{\mu}_b^{\mathrm{align}}\bigr\|_2^2.
\]
Unlike SCS, which measures feature-level difference, Drift captures whether the
global spatial focus of the representation moves under intervention.

All hidden-state metrics are computed at the evaluated feature block and then
averaged across the corresponding evaluation samples and timesteps.

\section{Random qualitative samples}
\label{app:qualitative_samples}

\begin{figure*}[ht]
    \centering
    \includegraphics[width=\linewidth]{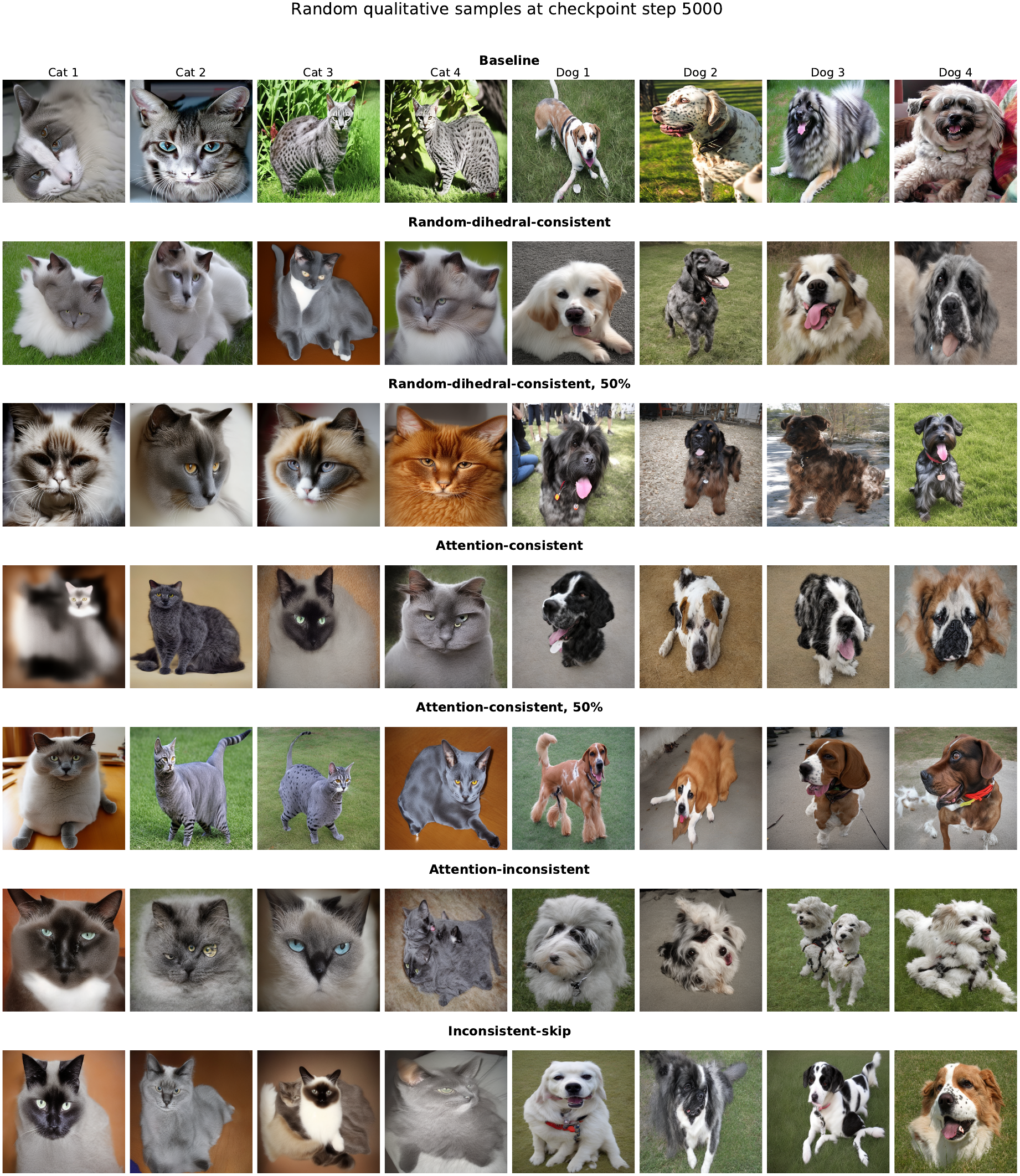}
    \caption{
Random qualitative samples from the final checkpoint ($5{,}000$ steps). The
examples were not manually curated; instead, they were sampled uniformly at
random from the full set of 105{,}000 generated images. For each mode, we
display four cat images and four dog images. Each row corresponds to one mode.
}
    \label{fig:qualitative-random-samples}
\end{figure*}

\section{Proof-of-Concept FID Comparison Details}
\label{app:proof_of_concept_details}

This appendix provides additional dataset, reference, and training details for
the proof-of-concept FID comparisons reported in
Sec.~\ref{sec:baseline_cifar_celeba_mnist}.

\subsection{Dataset Choice}
\label{app:proof_of_concept_datasets}

Because these experiments are intended as small-scale supporting comparisons
under limited computational resources, we restrict attention to compact and
well-established datasets. CIFAR-10 and CelebA-64 serve as the main U-Net
diffusion benchmarks, while MNIST provides a lightweight auxiliary
transformer-based diffusion setting.

\subsection{CIFAR-10 and CelebA-64 DDPM References}
\label{app:proof_of_concept_ddpm_refs}

For CIFAR-10 and CelebA-64, we use the reproduced vanilla DDPM baselines
reported by Proszewska \emph{et al.}~\cite{proszewska2025dmz}. For CIFAR-10,
their setup uses $32\times32$ images, a cosine noise schedule, batch size 128,
learning rate $10^{-4}$, and 250K training iterations, reporting FID 4.46 at
100 denoising steps. For CelebA-64, the same source reports a reproduced DDPM
baseline on $64\times64$ images with batch size 256, learning rate $10^{-4}$,
and 300K training iterations, yielding FID 4.51 at 100 denoising steps.

We use these reproduced values as our reference baselines because they provide
a consistent protocol across both CIFAR-10 and CelebA-64.

\subsection{Auxiliary Transformer-Based MNIST Reference}
\label{app:proof_of_concept_mnist_ref}

To complement the U-Net baselines, we include an auxiliary transformer-based
diffusion reference on MNIST using SDiT~\cite{yang2024sdit}. The paper reports
FID 5.54 on MNIST at $28\times28$ resolution, evaluated using 50{,}000 real and
50{,}000 generated images. We treat this result as an auxiliary
transformer-based reference rather than a directly matched DDPM baseline.

\subsection{Regularization Schedule}
\label{app:proof_of_concept_reg_schedule}

For all regularized variants, we keep the corresponding baseline architecture,
diffusion objective, optimizer, sampling procedure, and evaluation protocol
fixed, and add only our regularization term. In preliminary tuning, we found
that applying the regularizer too rarely left training too close to the
baseline, whereas applying it too frequently degraded performance. We therefore
apply the regularization to a randomly selected 50\% of mini-batches during
training. This schedule provided the most reliable trade-off between preserving
baseline dynamics and making the effect of the regularizer measurable in these
small-scale experiments. Accordingly, these results are intended as supporting evidence under limited
compute, rather than as definitive benchmark comparisons.


\section{Additional ViT and DiT Study Details}
\label{app:vit_dit_details_appendix}

This appendix provides brief setup details for the auxiliary ViT and DiT
studies reported in Sec.~\ref{sec:additional_studies}.

\subsection{ViT Fine-Tuning Setup}
\label{app:vit_setup_appendix}

For the ViT generalization study, we use a pretrained ViT-B/16
\cite{dosovitskiy2020vit} and fine-tune it on CIFAR-100
\cite{krizhevsky2009learning} for five epochs under three settings:
(i) standard input augmentation at $224\times224$,
(ii) standard input augmentation at $384\times384$, and
(iii) input-or-hidden augmentation, in which reflection-based transformations are
applied either to the input or to intermediate hidden states.

\subsection{Synthetic Probe Construction}
\label{app:synthetic_probe_appendix}

For the synthetic probe analysis, we use horizontal stripes, vertical stripes,
diagonal gradients, and Gaussian-like blobs on $4\times4$ and $8\times8$
grids. These probes provide controlled geometric patterns with different
symmetry properties, allowing us to test how consistently pretrained ViT
attention responds under horizontal reflection.

The probe consistency score is computed as
\[
S^{(\ell)}
=
\frac{1}{M}\sum_{i=1}^{M}
\mathrm{cosine}\!\left(A^{(\ell)}(x_i),A^{(\ell)}(T(x_i))\right),
\]
where $A^{(\ell)}$ denotes the selected attention field and $T$ is horizontal
reflection.

\subsection{DiT Synthetic Attention Comparison}
\label{app:dit_synthetic_appendix}

For the DiT analysis, we compare two interventions inside attention:
the output-only flip, which introduces QKV mismatch, and the QKV-consistent
transformation, which preserves a common spatial frame. We evaluate both on a
symmetric and a non-symmetric synthetic baseline to distinguish geometric
stability from geometric correctness under reflection. A minimal pseudocode implementation is given in
Appendix~\ref{app:dit_pseudocode_appendix}.

\section{Pseudo-code for the Flipped-Head Attention Intervention}
\label{app:dit_pseudocode_appendix}

Algorithm~\ref{alg:dit_flip} gives a minimal pseudocode view of the
flipped-head intervention in attention.

\begin{algorithm}[h]
\caption{Flipped-head intervention in multi-head attention}
\label{alg:dit_flip}
\begin{algorithmic}[1]
\Require query, key, value tensors $Q,K,V$; selected head $h^\star$; spatial transformation $T$
\For{each head $h$}
    \State $O_h \leftarrow \mathrm{softmax}(Q_hK_h^\top/\sqrt{d_k})V_h$
\EndFor
\State Reshape $O_{h^\star}$ to a spatial grid: $\mathcal{O}_{h^\star}\in\mathbb{R}^{P\times P\times d_v}$
\State Apply the transformation: $\widetilde{\mathcal{O}}_{h^\star}\leftarrow T(\mathcal{O}_{h^\star})$
\State Flatten back to token form: $\widetilde{O}_{h^\star}\in\mathbb{R}^{N\times d_v}$
\State Replace the selected head output: $O_{h^\star}\leftarrow \widetilde{O}_{h^\star}$
\State Concatenate all heads and project:
\[
Y\leftarrow \mathrm{Concat}(O_1,\dots,O_H)W^O
\]
\State \Return $Y$
\end{algorithmic}
\end{algorithm}

This pseudocode shows the output-only intervention used for the
attention-inconsistent case. In the attention-consistent variant, the
transformation is applied coherently to the interacting pathways before
attention is computed.

\subsection{Generalization to Other Heads}

Although Head~0 is used as the primary visualization head in this work, the
method is not restricted to it. The intervention can be applied to any head or
block, and its effects propagate through deeper layers regardless of which head
is selected for visualization.


\section{Rationale for Selecting Head 0}
\label{app:head0_appendix}

In all qualitative ViT experiments and visualizations, we report attention maps
using Head~0 of the selected transformer block (e.g., Block~5). This choice is
based on our preliminary qualitative analysis.

\subsection{Interpretability of Head~0}

In our preliminary qualitative inspection, Head~0 exhibited comparatively
compact and interpretable spatial patterns, making it a convenient choice for
visualization:
\begin{itemize}
    \item strong and compact activation on foreground objects,
    \item low background entropy,
    \item stable spatial localization across inputs,
    \item clear correspondence to human-interpretable regions.
\end{itemize}
This makes Head~0 particularly suitable for qualitative evaluation.


\end{document}